\documentclass[twoside]{article}

\usepackage[accepted]{aistats2026}
\usepackage{natbib}

\usepackage{amsmath}
\usepackage{amssymb}
\usepackage{mathtools}
\usepackage{amsthm}
\usepackage{thm-restate}
\usepackage{hyperref}
\theoremstyle{plain}
\newtheorem{theorem}{Theorem}[section]

\theoremstyle{definition}

\theoremstyle{remark}

\usepackage{dsfont}
\DeclareMathOperator*{\argmax}{arg\,max}
\DeclareMathOperator*{\argmin}{arg\,min}
\DeclareMathOperator*{\grad}{grad}
\DeclareMathOperator*{\support}{supp}

\usepackage{xparse}
\NewDocumentEnvironment{equations}{o}
{
  \subequations
  \IfValueTF{#1}
  {
    \label{equation:#1}
  }
  {}
  \align
}
{
  \endalign\endsubequations
}

\usepackage{tikz} 
\usetikzlibrary{automata,arrows,positioning,calc}

\usepackage{algorithm}
\usepackage{algpseudocode}
\algdef{SE}[VARIABLES]{Variables}{EndVariables}
   {\algorithmicvariables}
   {\algorithmicend\ \algorithmicvariables}
\algnewcommand{\algorithmicvariables}{\textbf{variables}}
\algdef{SE}[HYPERPARAMETERS]{Hyperparameters}{EndHyperparameters}
   {\algorithmichyperparameters}
   {\algorithmicend\ \algorithmichyperparameters}
\algnewcommand{\algorithmichyperparameters}{\textbf{hyperparameters}}

\begin{document}

\twocolumn[

\aistatstitle{An Information-Geometric Approach to Artificial Curiosity}

\aistatsauthor{ Alexander Nedergaard \And Pablo A. Morales }

\aistatsaddress{ Institute of Neuroinformatics \\ University of Zurich and ETH Zurich \And  Araya } ]

\begin{abstract}
Learning in environments with sparse rewards remains a fundamental challenge in reinforcement learning. Artificial curiosity addresses this limitation through intrinsic rewards to guide exploration, however, the precise formulation of these rewards has remained elusive. Ideally, such rewards should depend on the agent's information about the environment, remaining agnostic to its representation---an invariance central to information geometry. Leveraging this, we show that information monotonicity and invariance under the agent-environment interaction uniquely constrains intrinsic rewards to strictly concave functions of the reciprocal occupancy. Requiring these rewards to yield a principled exploration-exploitation trade-off, via information geodesic interpolation on the occupancy manifold, effectively limits the candidates to those determined by a scalar parameter. Remarkably, special values of this parameter are found to correspond to count-based and maximum entropy exploration. This framework provides important constraints to the engineering of intrinsic rewards while integrating foundational exploration methods into a single, cohesive model.
\end{abstract}

\section{Introduction}
Reinforcement learning systems adapt behavior through trial-and-error to maximize rewards. Inspired by the mammalian dopamine system, artificial reinforcement learning systems \citep{sutton2018reinforcement} have surpassed humans in many challenging domains \citep{mnih2015human,silver2016mastering,vinyals2019grandmaster,openai2019dota,kaufmann2023champion}. Behavior in reinforcement learning is understood as taking actions in response to observed states. In environments with many states, but sparse rewarding states, finding rewards becomes difficult. The problem of finding rewards, exploration, is a main obstacle to real-world application of general-purpose reinforcement learning methods.

Artificial curiosity \citep{schmidhuber1991possibility} proposes to use intrinsic rewards to guide exploration. However, the specific form of the intrinsic rewards is not constrained by the theory. In consequence, specific intrinsic rewards are typically grounded in empirical performance and anthropomorphic explanations such as curiosity or novelty seeking \citep{pathak2017curiosity,burda2018exploration}. Particular intrinsic rewards have been related through bounds to the theoretically grounded approach count-based exploration \citep{strehl2008analysis,kolter2009near,bellemare2016unifying}, which uses state counts in finite state spaces to guide exploration, notably achieving super-human performance in challenging boardgames \citep{silver2016mastering}. The recent maximum entropy exploration \citep{hazan2019provably} generalizes to infinite state spaces by maximizing entropy over states, with considerable empirical success under some practical deviations from its theoretical foundation \citep{liu2021behavior,nedergaard2022kmeans}, but has not been related to intrinsic rewards. 

Intrinsic rewards should ideally depend on the agents information about the environment, remaining agnostic to the representation of the information---an invariance central to information geometry, which concerns geometric structures invariant under information-preserving maps.
Classically~\citep{shannon1948mathematical}, such geometries are built upon the Kullback-Leibler (KL) divergence~\citep{kullback1951information}, but also exist for R\'enyi divergences~\citep{renyi1961measures} associated to a scalar characterizing the curvature of statistical manifolds~\citep{amari2000methods}.
Recent work in this direction has generalized the maximum entropy principle~\citep{jaynes1957information} to curved statistical manifolds with applications to various domains including neural networks~\citep{morales2021generalization,morales2023geometric,morales2023thermodynamics,aguilera2025explosive}, suggesting to investigate the role of such curvature in exploration.

In this paper, we prove that any representation-invariant intrinsic reward must be a strictly concave function of the reciprocal occupancy. A one-parameter subfamily of these functions uniquely yields a principled exploration–exploitation trade-off, via geodesic interpolation on the occupancy manifold. 
Moreover, count-based and maximum-entropy exploration arise as special cases of the scalar parameter, corresponding to distinct values of the occupancy manifold’s scalar curvature.
Our results are proven in the most general setting, establishing the role of these intrinsic rewards in the information-theoretic structure inherently associated to the reinforcement learning problem.

\section{Preliminaries}
Formally, we use Markov kernels prove our results for the most general state spaces, that is, second-countable measurable topological spaces, beyond which reinforcement learning problems cannot exist because the agent-environment interaction no longer forms a Markov chain \citep{meyn2012markov}. Furthermore, fundamental invariance results from reinforcement learning (Theorem \ref{theorem:ergodicity}) and information geometry (Theorem \ref{theorem:cencov}), which form the bedrock of our results, are naturally expressed using Markov kernels.

\subsection{Markov kernels}
Intuitively, Markov kernels extend the concept of transition matrices to spaces of uncountably many elements, such as continuous spaces. 
Formally, a Markov kernel from measurable space $(\mathcal{X},\Sigma)$ to measurable space $(\mathcal{Y},\Lambda)$ is a function $K:\mathcal{X}\times \Lambda \to [0,1]$ such that for every $x \in \mathcal{X}$, $K(x,\cdot)$ is a probability measure on $(\mathcal{Y},\Lambda)$, and every $\lambda \in \Lambda$, $K(\cdot,\lambda)$ is $\Lambda$-measurable, where $\cdot$ signifies $f(\cdot)\coloneqq x \mapsto f(x)$.
For brevity, we refer to $K$ as a Markov kernel from $\mathcal{X}$ to $\mathcal{Y}$. Each Markov kernel $K$ from $\mathcal{X}$ to $\mathcal{Y}$ induces a Markov morphism
\begin{equation}
    h_K(P) \coloneqq \int_\mathcal{X} K(x,\cdot) dP(x),
\end{equation}
which maps probability measures to probability measures. 
Informally, we may also view Markov kernels as operators that transform probability densities, although they technically operate on measures. A probability measure $P$ is said to be invariant under the Markov morphism $h$ if $h_K(P)=P$. For convenience, we often say that $P$ is invariant under $K$ to mean invariant under $h_K$.

\subsection{Reinforcement learning}
The heart of reinforcement learning is the agent-environment interaction, where the agent takes an action $a\in\mathcal{A}$ in response to an observed state $s\in\mathcal{S}$, and the environment changes state based on the state and action. Markov kernels formalize this most generally: The state space $\mathcal{S}$ and the action space $\mathcal{A}$ are measurable topological spaces, with $\mathcal{S}$ second-countable. The transition map $\delta$ is a Markov kernel from $\mathcal{S}\times\mathcal{A}$ to $\mathcal{S}$. The starting state probability $\mu$ is a probability measure in $\mathcal{S}$. The reward function $r$ is a map from $\mathcal{S}$ to $\mathbb{R}$. The policy $\pi$ is a Markov kernel from $\mathcal{S}$ to $\mathcal{A}$. The agent-environment interaction is the Markov kernel
\begin{equation}
  M(s,B) \coloneqq \int_{\mathcal{A}, B} d\delta(s,a,s') d\pi(s,a),\quad s\in\mathcal{S},B\in\mathcal{B},
\end{equation}
where $\mathcal{B}$ is the Borel $\sigma$-algebra of $\mathcal{S}$.
The reinforcement learning problem (Markov decision problem) is to find the optimal policy 
$\pi^* \coloneqq \argmax_{\pi\in\Pi} R(\pi)$ 
which maximizes the return
\begin{equation}
  R(\pi) \coloneqq \int_{\mathcal{S}^{n+1}} \sum_{i=0}^n r(s_i) d\mu(s_0)\prod_{i=1}^n dM(s_{i-1},s_i),
\end{equation}
where $n$ is the episode length. 
In the more general partially observable Markov decision problems, an observation function maps the state to an observation, and the policy instead maps these to actions.
We assume that the episode length does not depend on the policy. Practical deep reinforcement learning typically concerns episodic reinforcement learning ($n<\infty$) where the agent-environment interaction is guaranteed convergence to a steady state density, namely the occupancy (state visitation distribution). We prove this here due to some uncertainty in the reinforcement learning community regarding the result \citep{bojun2020steady}, emphasizing that the occupancy exists when the episodic reinforcement learning problem does and is uniquely invariant (stationary) under the agent-environment interaction:
\begin{restatable}{theorem}{ergodicity}
\label{theorem:ergodicity}
The Markov chain formed by the time-inhomogeneous Markov kernel
\begin{equation}
    M_t(s,B) \coloneqq \left\{
        \begin{aligned}
            & \mu(B), && \text{if $t$ divides $n+1$} \\ 
            & M(s,B), && \text {else} 
        \end{aligned} \right. 
\end{equation}
with $n\in\mathbb{Z}_+$, converges to a unique probability measure
\begin{equation}
    P_\pi(B) = \frac{1}{n+1}\int_\mathcal{S} \sum_{k=0}^n M^k(s,B) d\mu(s) 
\end{equation}
that is uniquely invariant under $M_t$.
\end{restatable}
Given a Borel measure $V$ on $\mathcal{S}$, we extend the conventional notion of occupancy by defining it either as the measure $P_\pi$ or as its Radon-Nikodym derivative
\begin{equation}
    p_\pi \coloneqq \frac{dP_\pi}{dV}.
\end{equation}
Intuitively, the occupancy is the generalization of normalized state counts to infinite spaces. For infinite state spaces, the occupancy can be estimated efficiently using a non-parametric density estimator \citep{hazan2019provably,mutti2020policy,nedergaard2022kmeans}.
Invariance of the occupancy under the agent-environment interaction provides a useful reformulation of the return in terms of the occupancy:
\begin{restatable}{lemma}{returnoccupancy}
\label{lemma:return_occupancy}
\begin{equation}
    R(p_\pi) \coloneqq (n+1)\int_\mathcal{S} p_\pi(s) r(s) dV(s) = R(\pi).
\end{equation}
\end{restatable}
The occupancy reflects the information the agent has about the environment; if the agent interacts with the environment for long enough, the agent's relative information about states eventually stabilizes, depending only on the relative frequency of state visits. The frequency of state visits determines the number of the samples the agent gathers from functions defined over states, such as the reward function, and more samples reduce the agent's uncertainty about those functions. Unsurprisingly, the occupancy forms the basis of some foundational exploration approaches.

\subsection{Exploration}
Artificial curiosity, count-based exploration and maximum entropy exploration are among the empirically strongest exploration approaches. Taking inspiration from biological curiosity, artificial curiosity \citep{schmidhuber1991possibility} proposes to add an intrinsic reward to the reward:
\begin{equation}
  r(s) + \beta \bar{r}(s), \quad \beta \in\mathbb{R}_{\geq 0}.
\end{equation}
Artificial curiosity with specific intrinsic rewards has had considerable empirical success \citep{pathak2017curiosity,burda2018exploration},
but such methods are typically heuristically motivated, making it difficult to formally reason about the form of $\bar{r}$. Furthermore, the very large set of functions $\{ \bar{r}:\mathcal{S}\rightarrow\mathbb{R}\}$ is not feasible to search over experimentally. Given the precedent of biological curiosity and evidence for surprise and novelty dopamine rewards in the brain \citep{kakade2002dopamine}, artificial curiosity is compelling but requires significant restrictions on the form of intrinsic rewards. A more formally grounded approach is count-based exploration \citep{kolter2009near,strehl2008analysis,bellemare2016unifying}, which is based on the principle of optimism in the face of uncertainty. The principle suggests that when estimating the expected value of a random reward function, we should be as optimistic in our estimation as our uncertainty permits. For subgaussian (e.g. bounded) reward functions, the Hoeffding bound \citep{hoeffding1962probability} shows that the uncertainty is bounded by the reciprocal square root of the number of samples. Count-based exploration proposes to add such an upper confidence bound to the reward:
\begin{equation}
  r(s) + \beta \frac{1}{\sqrt{n(s)}}, \quad \beta\in\mathbb{R}_{\geq 0},
\end{equation}
where $n(s)$ denotes number of visits to state $s$. The normalization of $n(s)$ coincides with the occupancy. Count-based exploration has been immensely successful on challenging problems with finite state spaces \citep{silver2016mastering}. A recent approach using the occupancy is maximum entropy exploration \citep{hazan2019provably}, which based on the principle of maximum entropy for exploration. The principle suggests that in the absence of observed rewards, states should be visited uniformly (c.f. \citet{jaynes1957information}). Using the result that Shannon entropy \citep{shannon1948mathematical}
\begin{equation}
    H(p) \coloneqq \int_\mathcal{S}p(s)\log\frac{1}{p(s)}dV(s) 
\end{equation}
is maximized by the uniform probability density $u$, maximum entropy exploration proposes to adds the Shannon entropy of the occupancy to the return:
\begin{equation}
    \label{equation:maximum_entropy}
    R(\pi)+\beta H(p_\pi), \quad \beta\in\mathbb{R}_{\geq 0}.
\end{equation}
The approach should not be confused with maximum entropy reinforcement learning~\citep{haarnoja19soft} which adds the Shannon entropy of the policy to the return. Practical maximum entropy exploration has had considerable empirical success when replacing the $\log$ in Shannon information with a different concave function to improve numerical stability under non-parametric density estimation \citep{liu2021behavior,nedergaard2022kmeans}.
We will show that this deviation from theory is not merely a heuristic, but provides a principled unification of count-based and maximum entropy exploration, through artificial curiosity with intrinsic rewards based on the occupancy. However, prior to this, we must establish that spaces of probability densities, such as the occupancy space, possess an invariant geometric structure.

\subsection{Information geometry}
Information geometry \citep{amari2000methods} studies parameterized families of probability densities $\left\{ p_\theta:\theta \in\mathbb{R}^d\right\}$ 
with additional geometric structure that is invariant under congruent Markov morphisms. 
Intuitively, congruent Markov morphisms are information-preserving maps, appearing as the equality condition in data processing inequalities such as \citep[Theorem 2.8.1]{cover2006elements}. 
We are interested in structures that are invariant under them, because these structures are agnostic to the representation of information.
We will refer to invariance under congruent Markov morphisms as information invariance, and the stronger property of satisfying a data processing inequality as information monotonicity. 
Formally, a Markov morphism is congruent if it has a left inverse induced by a statistic. Congruent Markov morphisms have a correspondence with sufficient statistics, as follows:
A statistic from $\mathcal{X}$ to $\mathcal{Y}$ is a measurable map $\kappa:\mathcal{X}\rightarrow\mathcal{Y}$ and induces a Markov morphism
\begin{equation}
  h_\kappa(P) \coloneqq \int_\cdot dP\kappa^{-1}(y)=P\kappa^{-1}.
\end{equation}
The statistic $\kappa$ is sufficient if there exists a Markov morphism $h$ in the opposite direction from $\mathcal{Y}$ to $\mathcal{X}$, such that $h$ is a right inverse of $h_\kappa$. Equivalently, the Markov morphism $h$ has a left inverse $h_\kappa$ induced by the statistic $\kappa$, making $h$ a congruent Markov morphism and $\kappa$ a sufficient statistic. The invariant geometric structures in information geometry are tensors, which for our purposes can be understood intuitively as mathematical objects that describe geometry in a coordinate-independent manner. A tensor $T$ is invariant under the Markov morphism $h$ if $h^*(T)=T$ where $h^*$ denotes the pullback of $h$. Pullbacks are maps between additional structure on sets, induced by maps between those sets. A fundamental result in information geometry concerns uniquely invariant tensors: 
\begin{theorem}
\label{theorem:cencov}
\citep{cencov1972statistical,ay2015information} The unique information invariant (0,2)-tensors and (0,3)-tensors are $\eta g$ and $\alpha T$, where $\eta,\alpha\in\mathbb{R}$, and
\begin{equation}
    g_{ij} (p_\theta) = \int_\mathcal{S} \partial_i \log p_\theta(s) \partial_j \log p_\theta (s) dP_\theta(s)
\end{equation}
is the Fisher-Rao tensor (Fisher information metric), and
\begin{equation}
    T_{ijk} (p_\theta) = \int_\mathcal{S} \partial_i \log p_\theta(s) \partial_j \log p_\theta (s) \partial_k \log p_\theta(s) dP_\theta(s)
\end{equation}
is the Amari-Čencov tensor.
\end{theorem}
where we adopt the notation $\partial_i\coloneqq\frac{\partial}{\partial \theta_i}$ and
$\partial'_i\coloneqq\frac{\partial}{\partial \theta'_i}$.

Intuitively, the Fisher-Rao and Amari-Čencov tensors describe a natural geometric structure that is invariant under information-preserving maps. However, this natural geometry is only determined up to the constants $\eta\in\mathbb{R}$ and $\alpha\in\mathbb{R}$, which respectively trivially scale and non-trivially asymmetrically deform the geometry.
Asymmetric geometry is already familiar to many machine learning researchers from the asymmetry of the KL-divergence~\citep{kullback1951information}.
An important generalization are the $f$-divergences \citep{renyi1961measures,csizar1963eine}:
\begin{equation}
    \mathcal{D}_f(p\|q) \coloneqq \int_\mathcal{S} p(s) f\left[\frac{q(s)}{p(s)}\right] dV(s)
\end{equation}
with $f$ is a thrice-differentiable strictly convex function satisfying $f(1)=1$, recovering the KL divergence for $f=-\log$.
An important example for us is
\begin{equation}
    f_{\alpha}(x) \coloneqq \frac{4}{1-\alpha^2}(1-x^\frac{\alpha+1}{2}), \quad \alpha \in \mathbb{R}
\end{equation}
which induces the $\alpha$-divergence~\citep{amari2000methods}
\begin{equation}
    \mathcal{D}_\alpha(p\|q) \coloneqq \frac{4}{1-\alpha^2}\left(1 - \int_\mathcal{S} p(s)^\frac{1-\alpha}{2}q(s)^\frac{\alpha+1}{2} dV(s)\right).
\end{equation}
Letting $\alpha = 2f''(1)+3f'''(1)$, which holds naturally for $\alpha$-divergences,
we see that $f$-divergences induce a geometry fundamentally connected to the information invariant tensors \citep{eguchi1983second}:
\begin{equations}
    \label{equation:eguchi_primal}
    \Gamma_{ijk}^f(p_{\theta}) 
    & \coloneqq -\partial_i \partial_j \partial_k^{'} \mathcal{D}_f(p_\theta \| p_{\theta'}) \Big|_{\theta = \theta'} \nonumber \\
    &= \Gamma_{ijk}^0(p_\theta) - \frac{\alpha}{2} T_{ijk}(p_\theta), \\
    \label{equation:eguchi_dual}
    \Gamma_{ijk}^{f^*}(p_{\theta})
    & \coloneqq -\partial_i \partial_j \partial_k^{'} \mathcal{D}_f(p_{\theta'} \| p_\theta) \Big|_{\theta = \theta'} \nonumber \\
    &=  \Gamma_{ijk}^0(p_\theta) + \frac{\alpha}{2} T_{ijk}(p_\theta), \\
    \label{equation:levi_civita}
    \Gamma_{ijk}^0(p_\theta) 
    &\coloneqq \frac{1}{2}(\partial_i g_{jk}(p_\theta) + \partial_j g_{ki}(p_\theta) - \partial_k g_{ij}(p_\theta)),
\end{equations}
where $\Gamma^f$, $\Gamma^{f^*}$ and $\Gamma^0=\frac{1}{2}(\Gamma^f+\Gamma^{f^*})$ are the coefficients of the primal, dual and Levi-Civita connections respectively. Connections determine geodesics, generalized straight lines, on curved manifolds. The $\alpha$-geodesics, or information geodesics, are the curves $\gamma:\mathbb{R}\rightarrow\mathcal{P}$ satisfying the geodesic equation: 
\begin{equation}
    \frac{d^2 \gamma_k}{dt^2} + \sum_{i,j}\Gamma_{ijk}^\alpha \frac{d\gamma_i}{dt} \frac{d \gamma_j }{dt} = 0.
\end{equation}
Intuitively, the Fisher-Rao tensor describes a Riemannian geometry with a symmetric divergence, and $\alpha$ determines its deformation by the Amari-Čencov tensor into a non-Riemannian geometry with an asymmetric divergence.
The geometry determined by $\alpha$ has constant sectional curvature $\frac{1}{4}(1-\alpha^2)$ \citep{amari2000methods}, so the geometry is spherical for $|\alpha|<1$, flat for $\alpha\pm1$ and hyperbolic for $|\alpha|>1$. With (\ref{equation:eguchi_primal}-\ref{equation:levi_civita}), the geometry is thus uniquely flat when $\alpha=\pm1$ and uniquely Riemannian when $\alpha=0$. For these special geometries, the $\alpha$-divergence corresponds to the KL-divergence ($\alpha=-1$), reverse KL-divergence ($\alpha=1$) and Hellinger divergence \citep{hellinger1909neue}  ($\alpha=0$).
Importantly, curvature depends on the space under consideration. For instance, the geometry is hyperbolic for $\alpha=0$ in the subspace of Gaussian probability densities, and the geometry is flat for any $\alpha\in\mathbb{R}$ in the ambient space of measures~\citep{amari2009alpha}. Closely related to this, $\alpha$-divergences satisfy the property that their gradients point along the geodesics of their induced geometry in the space of measures, a property which general $f$-divergences do not satisfy \citep[Proposition 2.12]{ay2017information}. Formalizing this, we say that a divergence is \textit{geodetic} if
\begin{equation}
    \grad \mathcal{D}(p\|\cdot)\big|_q = -\eta\left.\frac{d}{dt}{\gamma}_{q,p}(t)\right|_{t=0}
\end{equation}
with $\eta\in\mathbb{R}_{>0}$ and the $\alpha$-geodesic $\gamma_{q,p}$, connecting $q$ and $p$, induced by $\mathcal{D}$ in the space of measures. 
\begin{restatable}{lemma}{geodetic}
\label{lemma:geodetic}
The geodetic property of divergences is preserved by strictly monotonic functions.
\end{restatable}
The R\'enyi divergences \citep{renyi1961measures}
which are not $f$-divergences, are geodetic by a monotonic relation to $\alpha$-divergences through the Box-Cox transformation \citep{box1964analysis}. 
Monotonically related divergences $\mathcal{D}=F(\bar{\mathcal{D}})$ induce the same geometry (up to trivial scaling) through (\ref{equation:eguchi_primal}-\ref{equation:eguchi_dual}), and induce the same constants $\eta$ and $\alpha$ for the invariant tensors if the monotonic function $F$ additionally satisfies $F'(0)=1$. 
We now derive intrinsic rewards based on the geometric structures introduced here and show that artificial curiosity with these rewards unify foundational exploration approaches with the Amari-Čencov tensor's constant $\alpha$ as the sole degree of freedom.

\section{Invariant information rewards}
To formally ground artificial curiosity as rewarding information, we seek an intrinsic reward that depends on the information that the agent has about the environment and is agnostic to the representation of the information. Formalizing this, we now derive information rewards and prove that they are uniquely information monotonic and invariant under the agent-environment interaction.

\subsection{Measuring information}
In information theory, information is usually understood through Shannon entropy. For artificial curiosity, we want a notion of information that captures the information received by the agent when a state is observed, ideally without reference to the inner structure of the agent. A candidate function is Shannon information (surprisal) $I(s;p)\coloneqq-\log{p(s)}$ which can be interpreted as the surprise of observing a state $s$ given a subjective belief encoded by the probability density $p$. A geometric interpretation follows from the relation
\begin{equation}
    \underbrace{\int_\mathcal{S} p(s)I(s;p) dV(s)}_{=H(p)}-\log u = - \underbrace{\int_\mathcal{S} p(s)\log\frac{p(s)}{u}dV(s)}_{=\mathcal{D}_\text{KL}(p\|u)},
\end{equation}
where $u$ denotes the uniform probability density, such that Shannon information measures distinctness from maximal uncertainty along each state. Generalizing Shannon information to satisfy an analogous relationship with $f$-divergences, we define $f$-information as
\begin{equation}
    \label{eq:f_info}
    I_f(s;p) \coloneqq f\left[ \frac{1}{p(s)} \right]
\end{equation}
with $f$ a strictly concave function. $f$-Information has a similar geometric interpretation to Shannon information when $f(1)=0$ such that it induces an $f$-divergence $\mathcal{D}_{-f}$.  
The induced $f$-divergence is unchanged by transformations $f(x)\mapsto f(x)+c(x-1)$ with $c\in\mathbb{R}$. 
We say that a strictly concave function $f$ is geodetic if it induces a geodetic $f$-divergence.
It turns out that the functions $f_\alpha$, which induce the $\alpha$-divergences, are uniquely geodetic (c.f. \citet{amari2009alpha}):
\begin{restatable}{lemma}{alphainformationunique}
    \label{lemma:alphaunique}
    The family $\{f_\alpha:\alpha\in\mathbb{R}\} $ is uniquely geodetic, up to transformations $f_\alpha(x) \mapsto \eta \left\{ f_\alpha(x) + c(x-1) \right\}$ with $\eta,c\in\mathbb{R}$.
\end{restatable}
From these functions, we obtain $\alpha$-information
\begin{equation}
    I_\alpha(s;p) = \frac{4}{1-\alpha^2}\left(\Big[\frac{1}{p(s)}\Big]^\frac{\alpha+1}{2}-1 \right), \quad \alpha\neq1, 
\end{equation}
where, for notational simplicity we write $I_\alpha$ in place of $I_{f_\alpha}$. While $I_{\alpha\to 1}$ diverges, $I_{-1}\coloneqq I_{\alpha \to -1}$ converges to $-\log$. 
The distinctive geodetic property of $\alpha$‑information underlies a principled exploration-exploitation trade‑off, a fact that we will later prove. However, as we will now show, only strict concavity is required to ensure the invariance of $f$-information as an intrinsic reward in artificial curiosity---the fact which motivates our choice of definition.

\subsection{Artificial curiosity with information rewards}
The motivation that intrinsic rewards in artificial curiosity should depend on the information that the agent has about the environment, while remaining agnostic to the representation of the information, uniquely constrains intrinsic rewards to $f$-information relative to the occupancy (c.f. \cite{jiao2014information}):
\begin{restatable}{theorem}{uniquerewards} 
\label{theorem:unique_rewards}
Intrinsic rewards $\bar{r}(s;p) \coloneqq \bar{f}\left[ p(s) \right]$, with any function $\bar{f}:\mathbb{R}\rightarrow\mathbb{R}$ and any probability density $p$, are invariant under the agent-environment interaction $M$,
\begin{equation}
    \bar{r}(s;p) = \bar{r}(s;h_M(p)),
\end{equation}
uniquely when $p=p_\pi$. 
Furthermore, for intrinsic rewards $\bar{r}(s) \coloneqq \bar{r}(s;p_\pi)$, the intrinsic return
\begin{equation}
    \bar{R}(p_\pi) =\int_{\mathcal{S}^{n+1}} \sum_{i=0}^n \bar{r}(s_i) d\mu(s_0)\prod_{i=1}^n dM(s_{i-1},s_i)
\end{equation}
satisfies information monotonicity
\begin{equation}
    \bar{R}(p_\pi) \leq \bar{R}(h_\kappa(p_\pi)),
\end{equation}
with equality if and only if the statistic $\kappa$ is sufficient, uniquely when $\bar{r}(s)=I_f(s;p_\pi)$.
\end{restatable}
We refer to intrinsic rewards of the form $I_f(\cdot,p_\pi)$ as $f$-information rewards or simply information rewards. Artificial curiosity with $f$-information rewards has the form
\begin{equation}
    \label{equation:freward}
    r(s) + \beta I_f(s;p_\pi), \quad \beta\in\mathbb{R}_{\geq 0}.
\end{equation}
Crucially, the degrees of freedom in designing intrinsic rewards have been constrained to a strictly concave function $f$. Moreover, this restriction arises uniquely from the motivation of rewarding information. Further restrictions on $f$ will unify foundational exploration approaches with the Amari-Čencov tensor's constant $\alpha\in\mathbb{R}$ serving as the sole degree of freedom.

\subsection{Unifying exploration through geometry}
The degrees of freedom in artificial curiosity with information rewards~\eqref{equation:freward} are the strictly concave function $f$ and the parameter $\beta\in\mathbb{R}_{\geq0}$, which governs the exploration-exploitation trade-off by yielding a distinct optimal occupancy for each value. A principled method to interpolate between these occupancies is along $\alpha$-geodesics, which are uniquely invariant under information-preserving maps (Theorem~\ref{theorem:cencov}). Since $\alpha$-information functions are directly linked to these geodesics through their unique geodetic property (Lemma~\ref{lemma:alphaunique}), $\alpha$-information rewards provide a principled exploration-exploitation trade-off.
Remarkably, artificial curiosity with $\alpha$-information rewards specializes to count-based and maximum entropy exploration for the uniquely Riemannian and flat geometries:
\begin{restatable}{theorem}{equivalencecountbased}
\label{theorem:equivalence_count_based}
Artificial curiosity with
\begin{equation}
    r(s) + \beta I_\alpha(s;p_\pi), \quad \alpha\in\mathbb{R},\;\beta\in\mathbb{R}_{\geq 0}
\end{equation}
is equivalent to count-based exploration for $\alpha=0$ and maximum entropy exploration for $\alpha=-1$.
\end{restatable}
Equivalence is meant in the sense of having the same optimization objectives under gradients and constant scaling of $\beta$. 
The relationship to R\'enyi state entropy exploration~\citep{yuan2022renyi}, Boltzmann exploration~\citep{thrun1992efficient} and classic artificial curiosity~\citep{schmidhuber1991possibility} is given in Appendix \ref{section:weaker}.
Strikingly, count-based and maximum entropy exploration are usually derived from completely different principles and frameworks, but coincide with special values of the Amari-Čencov tensor constant $\alpha$ (see Figure~\ref{figure}).
\begin{figure}[t]
    \centering
    \includegraphics[width=0.85\linewidth]{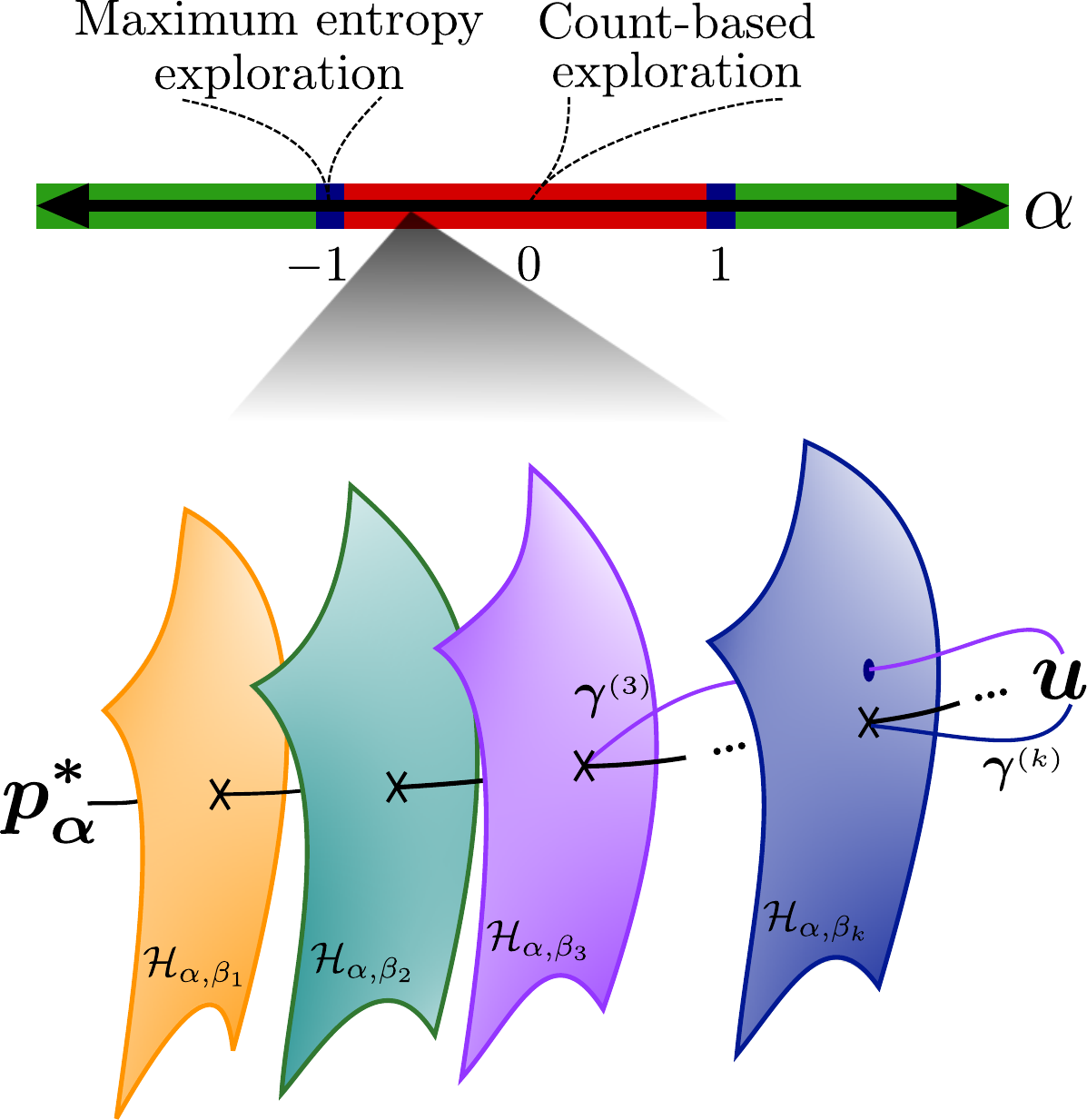}
    \caption{\textbf{Artificial curiosity with $\alpha$-information rewards on the curved occupancy manifold}. \textbf{(Top)}. The Amari-Čencov tensor constant $\alpha\in\mathbb{R}$ encodes the occupancy manifold curvature (red--spherical, blue--flat, green--hyperbolic). Count-based exploration corresponds to the Riemannian geometry with $\alpha=0$, and maximum entropy exploration to the flat geometry with $\alpha=-1$ (Theorem \ref{theorem:equivalence_count_based}). \textbf{(Bottom)} The intrinsic rewards scaling $\beta\in\mathbb{R}_{\geq 0}$ controls the exploration-exploitation trade-off on the curved occupancy manifold. The optima are $\alpha$-projections, along $(-\alpha)$-geodesics $\gamma^{(i)}$, from the uniform occupancy $u$ onto isoreturn hyperplanes $\mathcal{H}_{\alpha,\beta_i}$ (Proposition \ref{proposition:optima_projection}). The exploration-exploitation trade-off is $(\alpha+2)$-geodesic $\beta$-interpolation between the maximally rewarding occupancy $p_\alpha^*$ and the uniform occupancy $u$ (Theorem \ref{theorem:optima_geodesic}).}
    \label{figure}
\end{figure}
\section{Exploration on curved statistical manifolds} 
Artificial curiosity with $\alpha$-information rewards, exemplified by count-based and maximum entropy exploration, operates under the implicit assumption that the occupancy space constitutes a statistical manifold induced by the $\alpha$-divergence. An examination of exploration on this occupancy manifold reveals how these methods trade off exploration and exploitation in a principled manner. This geometric insight serves as the rationale for emphasizing $\alpha$-information rewards as a distinct and valuable subclass within the broader category of $f$-information rewards.
\subsection{A geometric exploration-exploitation trade-off}
\label{sec:geo_inf_Expl_TradeOff}
We now show that the optima of artificial curiosity with $\alpha$-information rewards uniquely trace geodesics on the occupancy manifold, reinforcing our claim that $\alpha$-information rewards yield a principled exploration-exploitation trade-off. Let us consider the optimization objective of artificial curiosity with information rewards:
\begin{equation}
\label{equation:augmented_return}
    \int_{\mathcal{S}^{n+1}} \sum_{i=0}^n \left\{ r(s_i)+\beta I_f(s_i;p_\pi) \right\} d\mu(s_0)\prod_{i=1}^n dM(s_{i-1},s_i).
\end{equation}
By Lemma \ref{lemma:return_occupancy}, we may consider the optimization objective as an equivalent function over the occupancy manifold: 
\begin{equation}
    R_{f,\beta}(p_\pi) 
    \coloneqq (n+1)\int_\mathcal{S}p_\pi(s)\left\{ r(s)+\beta I_f(s;p_\pi)\right\}dV(s).
\end{equation}The optima of artificial curiosity with $f$-information
\begin{equation}
    p_{f,\beta} \coloneqq \argmax_{p\in\mathcal{P}_\Pi}R_{f,\beta}(p)
\end{equation}
are points on the occupancy manifold. 
The set of occupancies $\mathcal{P}_\Pi\coloneqq\{p_\pi:\pi\in\Pi\}$ may not be the set of all probability densities $\mathcal{P}$, depending on the environment and policy class.
Because the return is linear in the occupancy, the set of occupancies where the same return is achieved, $\mathcal{H}(c) \coloneqq\{ p\in\mathcal{P}_\Pi:R(p)=c\}$ with $c\in\mathbb{R}$, is a hyperplane in the occupancy space. 
The optima of artificial curiosity with $\alpha$-information, $p_{\alpha,\beta}\coloneqq p_{f_\alpha,\beta}$, are the solutions to an $\alpha$-divergence minimization problem constrained to such a hyperplane:
\begin{restatable}{lemma}{optimadivergence}
\label{lemma:optima_divergence}
  If $P_\Pi$ is convex,
\begin{equation}
    p_{\alpha,\beta}=\argmin_{p\in\mathcal{H}_{\alpha,\beta}} \mathcal{D}_\alpha(p\|u),
\end{equation}
\end{restatable}
where the hyperplane $\mathcal{H}_{\alpha,\beta}\coloneqq\mathcal{H}(c_{\alpha,\beta})$ depends on $\alpha$ and $\beta$. The optima can now be understood in terms of pure geometry via geodesics:
\begin{restatable}{proposition}{optimaprojections}
    \label{proposition:optima_projection}
    If $P_\Pi$ is convex, the optima $p_{\alpha,\beta}$ are $\alpha$-projections from the uniform occupancy $u$ onto isoreturn hyperplanes $\mathcal{H}_{\alpha,\beta}$.
\end{restatable}
The $\alpha$-projections are the points on the hyperplanes where $(-\alpha)$-geodesics, connecting the uniform occupancy and the points, are orthogonal to $\alpha$-geodesics on the hyperplanes (see Figure~\ref{figure}).
Finding these points is equivalent to a generalized maximum entropy problem \citep{morales2021generalization} for which the solutions are curved exponential families:
\begin{restatable}{lemma}{optimaform}
\label{lemma:optima_form}
If $\mathcal{P}_\Pi=\mathcal{P}$,
\begin{equation}
  p_{\alpha,\beta} = \frac{(1+\frac{1}{\lambda\beta}r)^{-\frac{2}{\alpha+1}}}{\int (1+\frac{1}{\lambda\beta}r)^{-\frac{2}{\alpha+1}}dV}
\end{equation}
with $\lambda\in\mathbb{R}$ determined by $\alpha$ and $\beta$.
\end{restatable}
A concrete example of this is given in Appendix \ref{section:toy}. Importantly for optimization theory, all states are visited from these optima if $\beta>0$. As $\alpha$ changes, the isoreturn hyperplane and $\alpha$-projection change smoothly (c.f. \citet{amari2001information}). As $\beta$ changes from $0$ to $\infty$, the isoreturn hyperplane changes smoothly, and the $\alpha$-projections trace out a geodesic:
\begin{restatable}{theorem}{optimageodesic}
\label{theorem:optima_geodesic}
    If $\mathcal{P}_\Pi=\mathcal{P}$ and $\beta\geq\beta_0>0$, the exploration-exploitation relation $\beta\mapsto p_{f,\beta}$ is an $(\alpha+2)$-geodesic connecting the maximally rewarding occupancy $p_{\alpha,\beta_0}$ and the uniform occupancy $u$, uniquely when $f=f_\alpha$.
\end{restatable}

While this establishes the special role of $\alpha$-information rewards in the information-theoretic structure of the reinforcement learning problem, the result does not hold when states cannot be visited arbitrarily, i.e. $\mathcal{P}_\Pi\subset\mathcal{P}$.
Importantly for practice, the topology and geometry of the occupancy manifold determine whether optima can jump when lowering $\beta$:
\begin{restatable}{proposition}{optimapath}
  If $\mathcal{P}_\Pi$ is convex, the exploration-exploitation relation $\beta\mapsto p_{f,\beta}$ is a map. If additionally $\mathcal{P}_\Pi$ is compact, the map is continuous.
\end{restatable}
Since the observation function in partially observable Markov decision problems is a statistic, Theorem \ref{theorem:unique_rewards} shows that an upper bound on \eqref{equation:augmented_return} is optimized when using information rewards over observations instead of states. Therefore, Theorem \ref{theorem:optima_geodesic} holds for partially observable Markov decision problems if the observation function is sufficient, but not in general.

We have demonstrated that, under $\alpha$‑information rewards, artificial curiosity uniquely achieves a principled trade-off between exploration and exploitation: its optima coincide with $\alpha$‑projections, and the exploration-exploitation trade‑off is traced by ($\alpha+2$)‑geodesic $\beta$‑interpolation (see Figure~\ref{figure}). 
In the special case of maximum entropy exploration ($\alpha=-1$), the $\alpha$‑projection and $\beta$‑interpolation geodesics coincide. 
Moreover, these geometric insights extend beyond count‑based and maximum entropy exploration to R\'enyi state entropy exploration, once the trade‑off geodesic is reparameterized (see Appendix \ref{section:weaker}). Because the geodesic can be parameterized by R\'enyi entropy uniquely for $\alpha$-information rewards, it is only for these intrinsic rewards that $\beta$ has a natural information-theoretic interpretation as the \textit{information exploited by the optimal agent to maximize rewards}.
By Proposition \ref{proposition:optima_projection}, $\alpha$ has a geometric interpretation as the curvature of the occupancy manifold.
The importance of the occupancy manifold's curvature for exploration is evident from empirical performance gains when implicitly fine-tuning $\alpha$ to the environment \citep{yuan2022renyi}.
An interpretation based on $\alpha$-divergences is that negative $\alpha$ emphasizes the occupancy's peaks while positive $\alpha$ emphasizes the valleys, such that $\alpha$ controls a trade-off between penalizing high and low probability states (see Figure~\ref{figure:toy_solutions}), with these behaviors balanced by count-based exploration ($\alpha=0$). 
Discussion of the role of $\alpha$ for practical algorithms with density estimation can be found in Appendix \ref{section:algorithms}.

\section{Toy example}
\label{section:toy}
We here show a simple toy example to concretely instantiate concepts and illustrate the exploration-exploitation trade-off in artificial curiosity with $\alpha$-information rewards.
Consider a deterministic environment with two states and two actions, where the state switches when the second action is taken:
\begin{center}
	\begin{tikzpicture}[->, >=stealth', auto, semithick, node distance=3cm]
	\tikzstyle{every state}=[fill=white,draw=black,thick,text=black,scale=1]
	\node[state]    (A)               {$s_1$};
	\node[state]    (B)[right of=A]   {$s_2$};
	\path
	(A) edge[loop left]			node{$a_1$}	(A)
	(B) edge[loop right]		node{$a_1$}	(B)
	(A) edge[bend left,above]	node{$a_2$}	(B)
	(B) edge[bend left,below]	node{$a_2$}	(A);
	\end{tikzpicture}
\end{center}
Formally, we have a state space $\mathcal{S}=\{s_1,s_2\}$, action space $\mathcal{A}=\{a_1,a_2\}$ and transition probability 
\begin{equation}
  \delta(s_i, a_k, s_j) = \begin{cases}
    \mathbb{I}[i = j] & \text{ if } k = 1, \\
    \mathbb{I}[i \neq j] & \text{ if } k = 2,
  \end{cases}
\end{equation}
where $\mathbb{I}$ denotes the indicator function.
Suppose that only the second state is rewarding, $r(\{s_1,s_2\})=\{0,1\}$.
Let the episode length be infinite so that the occupancy does not depend on the starting state density.
Consider now a stochastic policy $\pi_\theta(s_i, a_1) = \theta_i$ parameterized by $\theta=\{ \theta_1, \theta_2 \}$ from the parameter space $\Theta=[0,1]^2$, such that the probability of switching from state $s_i$ is $\theta_i$.
Since the state and action spaces are finite, the agent-environment interaction
\begin{equation}
  M(s_i, \{s_j\})
  = \sum_{k=1}^{|\mathcal{A}|=2} \delta(s_i,a_k,s_j) \pi(s_i,a_k)
\end{equation}
is conveniently represented as the transition matrix
\begin{equation}
  M
  = \begin{bmatrix}
    1-\theta_1 & \theta_2 \\
    \theta_1 & 1-\theta_2
  \end{bmatrix}.
\end{equation}
Similarly, the occupancy can be represented as a column vector $p=\begin{bmatrix}p_1 &p_2\end{bmatrix}^\intercal$.
\begin{figure}%[ht!]
  \includegraphics[width=\linewidth]{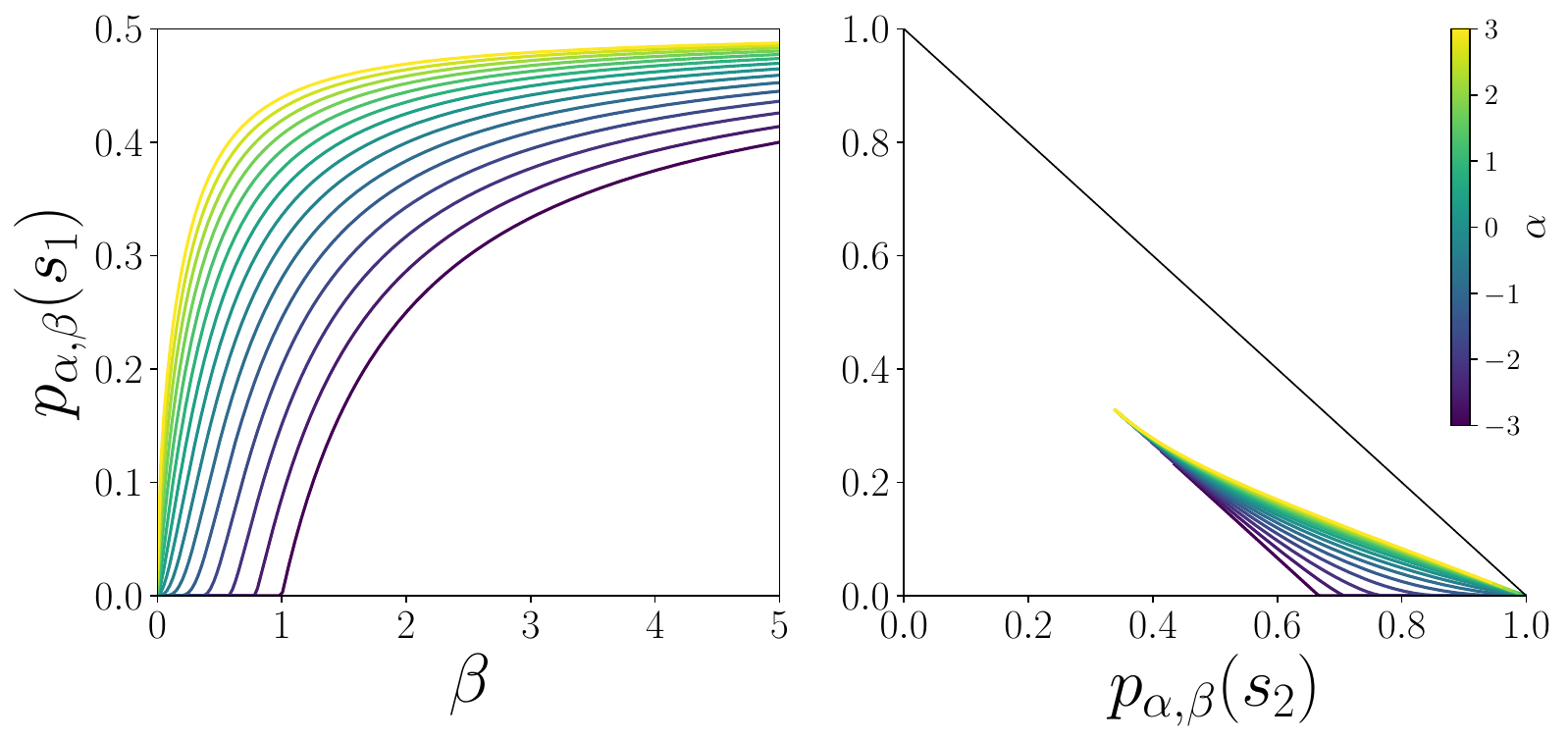}
  \caption{\textbf{Exploration-exploitation trade-off in artificial curiosity with $\alpha$-information rewards. (Left)}. The optimal occupancy $p_{\alpha,\beta}$ for different $\alpha$ and $\beta$ on a 2-state environment with reward function $r(\{s_1,s_2\})=\{0,1\}$. 
  When every occupancy is realizable, the optimal occupancies depend only on the state space and reward function.
   \textbf{(Right)}. The optimal occupancy for different $\alpha$ and $\beta$ on a 3-state environment with reward function $r(\{s_1,s_2,s_3\})=\{0,1,0.5\}$. The parameter $\beta$ controls interpolation between the uniform occupancy and the maximally rewarding occupancy, and the parameter $\alpha$ controls the balance between penalizing high and low probability states. Observe how lower $\alpha$ permits very low probability states, while higher $\alpha$ does not.}
  \label{figure:toy_solutions}
\end{figure}
From the unique invariance of the occupancy under the agent-environment interaction (Theorem \ref{theorem:ergodicity})
\begin{equation}
  \begin{bmatrix}
    1-\theta_1 & \theta_2 \\
    \theta_1 & 1-\theta_2
  \end{bmatrix}
\begin{bmatrix} p_1 \\ p_2 \end{bmatrix} =  \begin{bmatrix} p_1 \\ p_2 \end{bmatrix},
\end{equation}
we obtain a closed-form solution for the occupancy:
\begin{equation}
  p_i = \frac{\theta_j}{\theta_i + \theta_j}.
    \end{equation}
We note that it is unusual to have such a closed-form expression of the occupancy for practical problems.
By Lemma \ref{lemma:optima_form} and the particular reward function $r(\{s_1,s_2\})=\{0,1\}$, the probability of the first state at the optimal occupancy of artificial curiosity with $\alpha$-information rewards is given by
\begin{equation}
  p_{\alpha,\beta}(s_1) = \frac{1}{1 + (1 + \frac{1}{\lambda\beta})^{-\frac{2}{\alpha+1}}}
\end{equation}
with $\lambda\in\mathbb{R}$ determined by $\alpha$ and $\beta$. 
Note that the optimal occupancies $p_{\alpha,\beta}$ depend only on the state space and reward function if every occupancy is realizable. 
For the special case of $\alpha=-1$, we uniquely have a closed-form solution $p_{-1,\beta}(s_i) = e^{r(s_i) / \beta}$ by Lemma \ref{lemma:optima_divergence} and \citet{jaynes1957information}, such that we obtain a sigmoid function for the exploration-exploitation trade-off:
\begin{equation}
  p_{-1,\beta}(s_1) = \frac{1}{1 + e^{1/\beta}},
\end{equation}
By equating this to the occupancy, we see that optimal parameters for $\alpha=-1$ satisfy
\begin{equation}
\theta_1 = \frac{1+e^{1 / \beta}}{1+e^{-1 / \beta}} \theta_2.
\end{equation}
For $\alpha\neq-1$, we do not have a closed-form solution.
We used CVXPY \citep{diamond2016cvxpy,agrawal2019dgp} to compute the solutions and plotted them on the left in Figure~\ref{figure:toy_solutions}.
To illustrate the curvature of the occupancy space, encoded by $\alpha$, we extended the toy example to have a third state $s_3$ with reward $r(s_3)=0.5$ and plotted the optimal occupancies on the right in Figure~\ref{figure:toy_solutions}.

\section{Conclusion}
Information rewards---being the only intrinsic rewards that are information monotonic and invariant under agent-environment interactions---specialize artificial curiosity to strictly concave functions of the reciprocal occupancy. This invariance ensures that such rewards depend solely on the agent’s information about the environment, independent of any particular representation. In the special case of $\alpha$‑information, artificial curiosity uniquely attains a principled exploration-exploitation trade-off: its optima lie along the ($\alpha+2$)‑geodesic $\beta$‑interpolation on the occupancy manifold. Notably, the empirically successful count‑based ($\alpha=0$) and maximum entropy ($\alpha=-1$) exploration coincide with special values of the Amari-Čencov tensor's constant $\alpha$, highlighting how the curvature of the occupancy manifold governs effective exploration strategies. 
Furthermore, $\alpha=0$ and $\alpha=-1$ also suggest a broader interpretation within biological reinforcement learning, in which novelty and surprise can be understood as different points on a common $\alpha$-information spectrum (see Appendix ~\ref{appendix:novelty}).
Our findings unify foundational exploration approaches in a principled formalism and significantly constrain the design of intrinsic rewards for practical exploration in general state spaces.

\subsection{Refining the design of intrinsic rewards}
Deep reinforcement learning practitioners following our approach should use artificial curiosity with information rewards~\eqref{equation:freward}. Pseudocode for this is provided in Appendix~\ref{section:algorithms}. To apply information‑based intrinsic rewards in infinite state spaces, one must first estimate the occupancy density.  In practice, one may employ efficient non‑parametric estimators drawn from the maximum‑entropy exploration literature---e.g., kernel density estimation \citep{hazan2019provably}, k‑nearest‑neighbor density estimation~\citep{mutti2020policy} or k‑means density estimation~\citep{nedergaard2022kmeans}. Moreover, in many applications it is preferable to estimate occupancy over a transformed or lower‑dimensional representation of the state space. 
Although our present focus is on the formal structure of intrinsic rewards, efficient occupancy estimation and the choice of state space representation are essential for effective exploration in practice.  We therefore identify these topics---occupancy estimation and state space representation---as important avenues for future research.
The form of the strictly concave function $f$ remains unresolved and is primarily an empirical question to be investigated on meaningful reinforcement learning problems. 
The family $\left\{ f_\alpha : \alpha\in\mathbb{R} \right\}$, which induces $\alpha$-information, is promising as it uniquely yields a principled exploration-exploitation trade-off.
Furthermore, empirical studies have found superior performance on exploration problems when implicitly using $\alpha$-information rewards (see Appendix \ref{section:algorithms}).
Artificial curiosity with $\alpha$-information shares the same optima as R\'enyi state entropy exploration with the practical advantage of being computable locally in distributed reinforcement learning.
A promising future direction is practical algorithms using $\alpha$-information rewards with adaptive $\alpha$ and $\beta$.
In a separate publication, we will empirically investigate artificial curiosity with information rewards on robotics problems, including different methods for occupancy estimation and state space representation.
A thorough investigation of the effects of $f$, $\alpha$ and $\beta$ on optimization theory bounds, along with the deformation of statistical manifolds in reinforcement learning and scientific domains, has the potential to provide compelling theoretical justification for a specific information reward.

\section*{Acknowledgements}
P.A.M. acknowledges support by JSPS KAKENHI Grant Number 23K16855, 24K21518. A.N. acknowledges support by SNSF Grant Number 182539 (Neuromorphic Algorithms based on Relational Networks).
A.N. was a visiting researcher at ARAYA for a part of this work.

\bibliography{main}
\bibliographystyle{abbrvnat}

\clearpage
\appendix
\thispagestyle{empty}

\onecolumn
\aistatstitle{An Information-Geometric Approach to Artificial Curiosity: \\ Supplementary Materials}
\section{Practical algorithms}
\label{section:algorithms}
The theory of information rewards leads to a class of practical exploration algorithms with plug-and-play functionality for reinforcement learning and density estimation methods. In finite state spaces, occupancy estimation simply amounts to normalized state counts, while the occupancy must be efficiently estimated using non-parametric density estimation methods when the state space is infinite.
\begin{algorithm}
\caption{Reinforcement learning using artificial curiosity with information rewards.}
\begin{algorithmic}
\Hyperparameters
  \State information function $f$ (strictly concave)
  \State intrinsic reward scaling $\beta$ (non-negative)
\EndHyperparameters
\Variables
  \State transition function $\delta$
  \State starting state density $\mu$
  \State policy $\pi$
  \State occupancy estimate $\hat{p}_\pi$
\EndVariables
\For{\texttt{iteration}}
\State state $s \sim \mu$
\State trajectory $\tau \gets (s)$
\For{\texttt{step}}
  \State action $a \sim \pi(s)$
  \State next state $s' \sim \delta(s,a)$
  \State extrinsic reward $r_e \sim r(s)$
  \State intrinsic reward $r_i \gets f(\frac{1}{\hat{p}_\pi(s)})$
  \State $\tau \gets \tau.(a,r_e+\beta r_i,s')$
  \State $s \gets s'$
\EndFor
\State $\pi \gets \text{UpdatePolicy}(\pi, \tau)$
\State $\hat{p}_\pi \gets \text{UpdateDensityEstimate}(\hat{p}_\pi,\tau)$
\EndFor
\end{algorithmic}
\end{algorithm}
For the $k$-nearest neighbor \citep{mutti2020policy,liu2021behavior} and $k$-means \citep{nedergaard2022kmeans} density estimators from the maximum entropy exploration literature, the density estimate in a $n$-dimensional state space has the form
\begin{equation}
  \hat{p}_\pi(s) \equiv \frac{1}{m(s)^n}
\end{equation}
where $m(s)\in\mathbb{R}_{\geq 0}$ and $\equiv$ denotes equivalence up to affine transformations (such transformations $f(x)\mapsto \eta f(x)+c$ with $\eta,c\in\mathbb{R}$ do not affect policy gradients except for scaling them by $\eta$).
Using such a density estimate, information rewards have the form $I_f(s;\hat{p}_\pi)\equiv f\left[m(s)^n\right]$ with $f$ strictly concave. For maximum entropy exploration ($f=\log$), this leads to numerical instabilities when $m(s)$ is close to zero. To mitigate this, \cite{liu2021behavior} implicitly used $f(x)=\log(x+1)$ as a hack. However, this is a strictly concave function, so the hack is information-theoretically justified by our theory.
Similarly, \cite{nedergaard2022kmeans} reported best performance when implicitly using $f_\alpha$ with $\alpha=\frac{1}{n}-1$ through $\sqrt{m(s)}$, since
\begin{equation}
  I_\alpha(s;\hat{p}_\pi) \equiv \frac{4}{1-\alpha^2}(\left[ m(s)^n \right]^\frac{\alpha+1}{2}-1)
  \equiv \sqrt{m(s)}
\end{equation}
The theory of information rewards provides rigorous foundations for these empirically vital deviations from maximum entropy exploration theory.
When maximizing the Renyi entropy of the occupancy, \citet{yuan2022renyi} found best performance when implicitly fine-tuning $\alpha$ for each environment. Information rewards provide an efficient method of achieving such performance gains, since $\alpha$-information rewards have equivalent gradients but are computable in a distributed manner.
An interesting direction of future research are algorithms that use $\alpha$-information rewards with adaptive $\alpha$ and $\beta$, for example meta learning algorithms.

\section{Relation to non-equivalent methods}
\label{section:weaker}
Artificial curiosity with $\alpha$-information rewards is closely related to Rényi maximum entropy exploration \citep{yuan2022renyi}, Boltzmann exploration \citep{thrun1992efficient} and classic artificial curiosity \citep{schmidhuber1991possibility}. 
For Rényi maximum entropy exploration, which uses Rényi entropy
instead of Shannon entropy in~\eqref{equation:maximum_entropy},
we have:
\begin{restatable}{proposition}{equivalencerenyi}
  \label{proposition:equivalence_renyi_exploration}
  Artificial curiosity with $\alpha$-information is equivalent to Rényi maximum entropy exploration, under non-constant scaling of $\beta$.
\end{restatable}

From Lemma \ref{lemma:optima_form}, the optima of information rewards with $\alpha=-1$ are equivalent to Boltzmann exploration:
\begin{equation}
    \label{equation:boltzmann_exploration}
  \lim_{\alpha\rightarrow-1} p_{\alpha,\beta}(s) = \frac{e^\frac{r(s)}{\beta}}{ \int_\mathcal{S} e^\frac{r}{\beta} dV}
\end{equation}
The optima of Rényi maximum entropy exploration are similarly equivalent by Proposition \ref{proposition:equivalence_renyi_exploration}.

By extending $\alpha$-information rewards with a discrete parameter, it is possible to include the classic artificial curiosity of \citet{schmidhuber1991possibility}. 
In particular, consider artificial curiosity with $\alpha$-information relative to the $k$-step agent-environment interaction
$M^k$ from some state $s_0$: 
\begin{equation}
  \label{equation:generalized_alpha_curiosity}
  r(s) + \beta I_\alpha(s;M^k\left[s_0;\cdot\right]), \quad \alpha\in\mathbb{R}, \; \beta\in\mathbb{R}_{\geq0}, \; k\in\mathbb{Z}_+, \; s_0 \in\mathcal{S}.
\end{equation}
By definition of the occupancy, we have
\begin{equation}
  \lim_{k\rightarrow\infty} I_\alpha(s;M^k\left[s_0;\cdot\right]) = I_\alpha(s;p_\pi)
\end{equation}
independent of $s_0$, such that \eqref{equation:generalized_alpha_curiosity} contains count-based $(k=\infty,\alpha=0)$ and maximum entropy exploration ($k=\infty,\alpha=-1$) by Theorem \ref{theorem:equivalence_count_based}. 
Furthermore, we have
\begin{restatable}{proposition}{equivalenceclassic}
  \label{proposition:classic_artificial_curiosity}
  If the agent-environment interaction is a unit variance Gaussian with mean $f_\theta(s_0)$, \eqref{equation:generalized_alpha_curiosity} is equivalent to classic artificial curiosity when $k=1$ and $\alpha=-1$.
\end{restatable}
Count-based and maximum entropy exploration tend to empirically outperform classic artificial curiosity, suggesting that agent-environment interaction invariance is associated with superior empirical performance.

\section{Novelty and surprise are on a spectrum}
\label{appendix:novelty}
Because information-processing in brains and machines are bound by the same mathematical structure, including the unique information-theoretic role of $\alpha$-information, our results have implications for neuroscience. The firing rates of dopamine neurons in the primate ventral tegmental area are correlated with reward prediction errors \citep{schultz1997neural}, but also with novel and surprising stimuli \citep{kakade2002dopamine}, suggestive of a neurobiological basis of curiosity in mammals. Several authors have argued for a functional and mechanistic distinction between dopaminergic novelty and surprise rewards \citep{barto2013novelty,xu2021novelty,modirshanechi2023surprise}. However, our results suggest viewing novelty and surprise as $\alpha$-information on a continuous spectrum determined by $\alpha$. In particular, $0$-information of neural populations representing generative models corresponds to novelty, and $(-1)$-information of neural populations representing predictive models corresponds to surprise (if the model is Gaussian, prediction errors). Intriguingly, empirical evidence suggests that human exploratory behavior is best explained by novelty rewards \citep{xu2021novelty}. Experimental investigations are warranted on whether $\alpha$-information rewards explain biological curiosity data better than other theories, and if so, whether inferred parameters for $\alpha$ and $\beta$ are consistent with empirical and theoretical results from reinforcement learning.

\section{Proofs}
\ergodicity*
\begin{proof}
This result is well-known, but we did not find a proof for general state spaces in the literature.
We first construct a time-inhomogeneous Markov kernel equivalent to $M_t$ and show that it forms a Markov chain. To do this, we use a counter set $I\coloneqq\{0,...,n\}$ and define:
\begin{equation}
   \tilde{\mathcal{S}} \coloneqq \mathcal{S} \times I
\end{equation}
\begin{equation}
  \tilde{\mu}(B) \coloneqq (\mu(B), 0)
\end{equation}
\begin{equation}
    \label{equation:transition_kernel}
    \tilde{M}(s,B) \coloneqq 
    \left\{
    \begin{aligned}
        & (\tilde{\mu}(B), 0) & \text{if $c(s) = n$,} \\
        & (M(s,B), c(s) + 1) & \text{else,}
    \end{aligned} 
    \right.
\end{equation} 
where $c(s)$ gives the counter and $B\in\tilde{B}$. To ensure that $\tilde{M}$ forms a Markov chain given that $M$ does, we need to handle a topological technicality: $\tilde{B}$, the Borel $\sigma$-algebra of $\tilde{S}$, must be countably generated. 
Let $I$ have any topology and let $\tilde{\mathcal{S}}$ have the product topology. If $M$ formed a Markov chain, $\mathcal{S}$ must have been second-countable. Since $I$ is second-countable by finiteness, and the property is preserved by the product topology, $\tilde{\mathcal{S}}$ is second-countable. $\tilde{\mathcal{B}}$ is then countably generated and $\tilde{M}$ forms a Markov chain.

We now construct $P_\pi$ and show that it is the unique invariant probability measure for the Markov chain.
Let
\begin{equation}
    \label{equation:ergodicity_avoidance}
    \tilde{M}_\mathrm{X}^k(s,B) \coloneqq \int_{\mathcal{S}\setminus\mathrm{X}} \tilde{M}_\mathrm{X}^{k-1}(s,s')\tilde{M}(s',B) dV(s'), \quad \tilde{M}_\mathrm{X}^1(s,B)\coloneqq \tilde{M}(s,B)
\end{equation}
denote the probability of visiting $B\in\tilde{\mathcal{B}}$ from $s\in\mathcal{S}$ in $k$ steps without passing through $\mathrm{X}\subseteq\tilde{\mathcal{B}}$. Let 
\begin{equation}
    \Xi\coloneqq\{ s \in \tilde{\mathcal{S}} : c(s) = n\}
\end{equation} 
denote the terminal states. Now, define the expected number of visits to $B$ without passing through a terminal state:
\begin{equation}
    \tilde{P}_\pi(B) \coloneqq \int_{\tilde{\mathcal{S}}} \sum_{k=0}^\infty \tilde{M}_\Xi^k (s,B) d\mu(s)
\end{equation}
Since we are guaranteed to visit $\Xi$ in exactly $n$ steps from a starting state, we have
\begin{equations}
    \tilde{P}_\pi(B) &= \int_{\tilde{\mathcal{S}}} \Big( \sum_{k=0}^{n} \tilde{M}_\Xi^k (s,B) + \sum_{k=n+1}^\infty \tilde{M}_\Xi^k (s,B)\Big) d\mu(s) \\
    \label{equation:ergodicity_finite}
    &= \int_{\tilde{\mathcal{S}}} \sum_{k=0}^{n} \tilde{M} (s,B) d\mu(s) \ge 0
\end{equations}
since only the first $n$-terms of the sum contribute. We can normalize $\tilde{P}_{\pi}$ to get a probability measure:
\begin{equations}
    P_\pi(B) &= \frac{\tilde{P}_\pi(B)}{\tilde{P}_\pi(\mathcal{S})} = \frac{\int_{\tilde{\mathcal{S}}} \sum_{k=0}^{n} \tilde{M}^k (s, B) d\tilde{\mu}(s)}{\int_{\tilde{\mathcal{S}}} \sum_{k=0}^{n} \tilde{M}^k (s, \mathcal{S}) d\tilde{\mu}(s)} \\
    \label{equation:ergodicity_inhomogenous}
    &= \frac{1}{n+1}\int_{\tilde{\mathcal{S}}} \sum_{k=0}^{n} \tilde{M}^k (s, B) d\tilde{\mu}(s)
\end{equations}
$\tilde{P}_\pi$-irreducibility and recurrence of the Markov chain implies the existence of a unique, up to constant multiplies, invariant measure equivalent to $\tilde{P}_\pi$ \citep[Theorem 10.4.9]{meyn2012markov}. Therefore, $P_\pi$ is the unique invariant probability measure if the chain is $\tilde{P}_\pi$-irreducible and recurrent. 
Aperiodicity is not necessary because we have already constructed the measure as a finite sum. By \citet[Proposition 4.2.1]{meyn2012markov}, $\tilde{P}_\pi$-irreducibility and recurrence follows from
\begin{equation}
    \forall s \in \mathcal{S}. \; \forall B \in \tilde{\mathcal{B}}. \; \tilde{P}_\pi(B) > 0 \Rightarrow \sum_{k=1}^\infty \tilde{M}^k(s,B) = \infty \
\end{equation}
Since the state resets every $n+1$ steps, we have for any $k$ and $m\in\mathbb{N}$,
\begin{equation}
    \label{equation:ergodicity_modulo}
    \tilde{M}^k(s,B) = \tilde{M}^{k+m(n+1)}(s,B)
\end{equation}
Letting $l \coloneqq n - c(s)$ denote the number of steps until the next state reset, we furthermore have for $k > l$,
\begin{equations}
    \tilde{M}^k(s,B) 
    &= \int_S \tilde{M}^{k-(l+1)}(s',B) d\tilde{M}^{l+1}(s,s') \\
    \label{equation:ergodicity_reset}
    &= \int_S \tilde{M}^{k-(l+1)}(s',B) d\tilde{\mu}(s')
\end{equations}
since a state reset happens after $l$ steps. Then,
\begin{equation} 
    \sum_{k=1}^\infty \tilde{M}^k(s,B) 
    = \underbrace{\sum_{k=1}^{l+1} \tilde{M}^k(s,B)}_{\geq 0} + \sum_{k=l+1}^\infty \tilde{M}^k(s,B) \geq \sum_{k=l+1}^\infty \tilde{M}^k(s,B),
\end{equation}
as a bound from below,
\begin{equations}
    \sum_{k=l+1}^\infty \tilde{M}^k(s,B) & \overset{(\ref{equation:ergodicity_reset})}{=} \sum_{k=l+1}^\infty \int_{\tilde{\mathcal{S}}} \tilde{M}^{k-(l+1)}(s',B) d\tilde{\mu}(s') \\ 
    & = \sum_{k=0}^\infty \int_{\tilde{\mathcal{S}}}  \tilde{M}^k(s',B) d\tilde{\mu}(s') \\  
    & = \sum_{m=0}^\infty \sum_{k=0}^n \int_{\tilde{\mathcal{S}}}  \tilde{M}^{k+m(n+1)}(s',B) d\tilde{\mu}(s') \\  
    & \overset{(\ref{equation:ergodicity_modulo})}{=} \sum_{m=0}^\infty \sum_{k=0}^n \int_{\tilde{\mathcal{S}}}  \tilde{M}^{k}(s',B) d\tilde{\mu}(s') \\  
    & = \lim_{m\rightarrow\infty} m \int_{\tilde{\mathcal{S}}} \sum_{k=0}^{n} \tilde{M}^k(s',B) d\tilde{\mu}(s') \\  
    & \overset{(\ref{equation:ergodicity_finite})}{=} \lim_{m\rightarrow\infty} m \tilde{P}_\pi(B) = \begin{cases}
    \infty & \text{ if } \tilde{P}_\pi(B) > 0, \\
    0 & \text{ else.}
    \end{cases} \label{equation:ergodicity_irreducibility}
\end{equations} %align}
The Markov chain is then $\tilde{P}_\pi$-irreducible and recurrent, and $P_\pi$ is the unique invariant probability measure under $\tilde{M}$. Since $M=\tilde{M}$ when $c(s)<n$, and $c(s)=0$ for any $s\in\support (\tilde{\mu})$, $P_\pi$ is invariant under $M$ and it follows from  (\ref{equation:ergodicity_inhomogenous}) that
\begin{equation}
  P_\pi(B) = \frac{1}{n+1}\int_{\mathcal{S}} \sum_{k=0}^{n} M^k (s, B) d\mu(s).
\end{equation}
\end{proof}
\returnoccupancy*
\begin{proof}
The proof is due to \citep{bojun2020steady}.
Define the state-action value
\begin{equation}
    Q(s,a) \coloneqq \int_{\mathcal{S}^{n-k}} \sum_{i=k}^n r(s_i) d\delta(s,a,s_k)dM(s_k,s_{k+1})\cdots dM(s_{n-1},s_n)
\end{equation}
where the notation $\int_{\mathcal{S}^{n-k}} \coloneqq \prod^{n-k} \int_\mathcal{S}$ had been adopted. $Q(s,a)$ satisfies,
\begin{equations}
    R(\pi) &= \int_\mathcal{S} \left\{ r(s) + \int_\mathcal{A} Q(s,a) d\pi(s,a)\right\} d\mu(s) \\
    \label{equation:return_occupancy_return}
    &=  \int_\mathcal{A} Q(\xi,a) d\pi(\xi,a)
\end{equations}
for any terminal state $\xi\in\Xi$.
The Bellman equation \citep{bellman1954theory} shows that
\begin{equation} \label{equation:bellman}
    Q(s,a) = \int_\mathcal{S} \left\{ r(s') +  \gamma(s') \int_\mathcal{A} Q(s',a') d\pi(s',a')  \right\} d\delta(s,a,s')
\end{equation}
where
\begin{equation}
    \gamma(s) \coloneqq
    \begin{cases}
        1 & \text{ if $s \notin \Xi$,} \\
        0 & \text{ else.}
    \end{cases}
\end{equation}
Integrating both sides of~\eqref{equation:bellman} over $dP_\pi$ and $d\pi$, we have
\begin{equations}
    \int_{\mathcal{S}\times \mathcal{A}} Q(s,a) d\pi(s,a) dP_\pi(s) 
    &= \int_{\mathcal{A}\times \mathcal{S}^2} \left\{ r(s') + \gamma(s') \int_\mathcal{A} Q(s',a') d\pi(s',a') \right\} \nonumber \\
    &\qquad \qquad \;\; d\delta(s,a,s') dP_\pi(s) d\pi(s,a) \\  
    &= \int_{\mathcal{S}^2} \left\{ r(s') + \gamma(s') \int_\mathcal{A} Q(s',a') d\pi(s',a') \right\} \nonumber \\
    &\qquad \quad \;\, dM(s,s') dP_\pi(s) \\
    &= \int_\mathcal{S} \left\{ r(s) + \gamma(s) \int_\mathcal{A} Q(s,a) d\pi(s,a)\right\} dP_\pi(s) \\
    &= \int_\mathcal{S} r(s) dP_\pi(s) + \int_\mathcal{S} \gamma(s) \int_\mathcal{A} Q(s,a) d\pi(s,a) dP_\pi(s)
\end{equations} 
where we used the invariance of $P_\pi$ by Theorem \ref{theorem:ergodicity} in the second-last line. Subtracting the second term on the right-hand side from both sides,
\begin{equations}
    \int_\mathcal{S} r(s) dP_\pi(s)
    &= \int_\mathcal{S} \left\{ [1-\gamma(s)] \int_\mathcal{A} Q(s,a) d\pi(s,a) \right \}dP_\pi(s) \\
    &= \int_{\Xi \times \mathcal{A}} Q(s,a) d\pi(s,a) dP_\pi(s) \\
    &\overset{(\ref{equation:return_occupancy_return})}{=} R(\pi)\int_\Xi dP_\pi(s) = \frac{1}{n+1}R(\pi)
\end{equations}
this finally implies that
\begin{equations}
  R(\pi) &= (n+1) \int_\mathcal{S} r(s) dP_\pi(s) \\
  &= (n+1) \int_\mathcal{S} p_\pi(s) r(s) dV(s).
\end{equations}
\end{proof}

\geodetic*
\begin{proof}
Let $\bar{\mathcal{D}}$ be a geodetic divergence and $\mathcal{D}=F(\bar{\mathcal{D}})$ with $F$ strictly monotonic function.
Using the shorthand $\mathcal{D}_p \coloneqq \mathcal{D}(p\|\cdot)$, letting $g$ denote the Fisher-Rao tensor, and letting $\gamma$ denote a unit speed path with arbitrary endpoints $p$ and $q$,
\begin{align}
    \bar{\mathcal{D}}_p(q) &= \bar{\mathcal{D}}_p(q) - \underbrace{\bar{\mathcal{D}}_p(p)}_{=0} \nonumber \\
    &= \int_\gamma d\bar{\mathcal{D}}_p =\int_\gamma g_{\gamma}(\grad \bar{\mathcal{D}}_p, \gamma')
\end{align}
and
\begin{align}
    \bar{\mathcal{D}}_p(q) 
    &= \int_\gamma d\bar{\mathcal{D}}_p 
    = \int_\gamma d(F(\mathcal{D}_p)) 
    = \int_\gamma F'(\mathcal{D}_p)d\mathcal{D}_p \nonumber \\
    &= \int_\gamma F'(\mathcal{D}_p) g_{\gamma}(\grad \mathcal{D}_p, \gamma') 
    = \int_\gamma g_{\gamma}(F'(\mathcal{D}_p) \grad \mathcal{D}_p, \gamma')
\end{align}
where $\int_\gamma$ denotes a line integral. Since it then holds for any unit speed path $\gamma$ with arbitrary endpoints $p$ and $q$ that
\begin{equation}
    \int_\gamma g_{\gamma}(\grad \bar{\mathcal{D}}_p, \gamma') = \int_\gamma g_{\gamma}(F'(\mathcal{D}_p) \grad \mathcal{D}_p, \gamma')
\end{equation}
it must be that
\begin{equation}
    \grad \bar{\mathcal{D}}(p,\cdot) = F'(\mathcal{D}(p,\cdot)) \grad \mathcal{D}(p,\cdot)
\end{equation}
for any $p$. Since additionally $F'(\mathcal{D}(p\|\cdot))>0$ by strict monotonicity of $F$, $\mathcal{D}$ must be geodetic if $\bar{\mathcal{D}}$ is.
\end{proof}

\alphainformationunique*
\begin{proof}
Recall that a function $f$ is geodetic if its induced divergence $\mathcal{D}_f$ satisfies:
\begin{equation}
  \grad \mathcal{D}_f(p \|\cdot) \big|_q = -\eta\left.\frac{d}{dt}\gamma(t) \right|_{t=0}
\end{equation}
where $\eta\in\mathbb{R}_{>0}$ and $\gamma$ is the primal geodesic, connecting $q$ and $p$, induced by $\mathcal{D}_f$ in the space of measures.
By \citet[Equation 2.59]{ay2017information}, $\gamma$ is given by
\begin{equation}
    \gamma(t) = \left\{ (1-t) q^{\frac{1-\alpha}{2}} + t p^{\frac{1-\alpha}{2}} \right\}^{\frac{2}{1-\alpha}}
\end{equation}
such that
\begin{equations}
    \frac{d}{dt}\gamma(t) \Big|_{t=0} 
    &= \frac{2}{1-\alpha} \left[ (1-t)q^\frac{1-\alpha}{2} + tp^\frac{1-\alpha}{2} \right]^\frac{1+\alpha}{1-\alpha}(p^\frac{1-\alpha}{2} - q^\frac{1-\alpha}{2}) \Big|_{t=0} \\
    &= \frac{2}{1-\alpha}q^\frac{1+\alpha}{2}(p^\frac{1-\alpha}{2} - q^\frac{1-\alpha}{2}) \\
    &= \frac{2}{1-\alpha}q\bigg(\left[ \frac{p}{q} \right]^\frac{1-\alpha}{2}-1\bigg)
\end{equations}
We proceed by equating this to the gradient of the $f$-divergence and showing that this uniquely constrains $f$.
Some technical machinery is required to handle the gradient on a statistical manifold when the underlying state space is infinite. We follow \citet{ay2017information} and consider a 2-integrable parameterized measure model $(\Theta,\mathcal{S},\mathbf{p})$ with a smooth map $\mathbf{p}: \Theta \rightarrow\mathcal{M}$ between the Banach manifolds of parameters $\Theta$ and measures $\mathcal{M}$ over the state space $\mathcal{S}$.
The Fisher-Rao tensor on $\Theta$, $g$, is then rigorously defined as the pullback of a natural metric tensor on $\mathcal{M}^{1/2}$ along the smooth map $\sqrt{\mathbf{p}}:\Theta\rightarrow\mathcal{M}^{1/2}$:
\begin{equation}
    \label{equation:alphainformationunique_natural}
  g_\theta(v,w) = 4\int_\mathcal{S}d(d_\theta\sqrt{\mathbf{p}}(v)  d_\theta\sqrt{\mathbf{p}}(w))
\end{equation} 
Letting $\partial_v$ denote the Gateaux derivative with respect of to $v\in T_\theta\Theta$, and letting the Radon-Nikodym derivative $q = \frac{d\mathbf{p}(\theta)}{dV}$, the gradient of $\mathcal{D}_f(p\|q)$ at $\theta$ is characterized similarly to on Riemannian manifolds:
\begin{equation}
    \label{equation:alphainformationunique_gradient}
    \partial_v \mathcal{D}_f(p\|q)
    = g_\theta(v,\grad \mathcal{D}_f(p\|q))
\end{equation}
Now, observe that
\begin{equations}
    \partial_v D_f(p\|q)
    &= \partial_v \int_\mathcal{S} p(s) f\left[ \frac{q(s)}{p(s)} \right] dV(s) \\
    &= \int_\mathcal{S} \left\{ \partial_v q(s) \right\} f' \left[ \frac{q(s)}{p(s)} \right] dV(s) \\
    &= 4 \int_\mathcal{S} d \left( \left\{ \frac{\sqrt{V(s)}}{2\sqrt{q(s)}} \partial_v q(s) \right\} \left\{ \frac{\sqrt{V(s)}}{2\sqrt{q(s)}} q(s) f' \left[ \frac{q(s)}{p(s)} \right] \right\} \right) \\
    &= 4\int_\mathcal{S}d\left(d_\theta\sqrt{\mathbf{p}}(v)  d_\theta\sqrt{\mathbf{p}}(w)\right)
\end{equations}
with $w\in T_\theta\Theta$ satisfying $d_\theta\mathbf{p}(w)=q f'\left[ \frac{q}{p}\right]V$. Comparing to~\eqref{equation:alphainformationunique_natural} and~\eqref{equation:alphainformationunique_gradient}, we see that $w=\grad D_f(p\|q)$ on $\Theta$ and that the Radon-Nikodym derivative (with respect to $V$) of its pushforward to $\mathcal{M}$ (along $\mathbf{p}$) is given by $q f'\left[ \frac{q}{p}\right]$. 
We need the gradient on $\mathcal{M}$, rather than $\Theta$, in order to relate it to the derivative $\frac{d}{dt}\gamma(0)\in T_{\gamma(0)}\mathcal{M}$.
Now, suppose that the function $f$ is geodetic, then
\begin{equation}
    q f' \left[ \frac{q}{p} \right]
    = - \eta \frac{2}{1-\alpha}q \bigg( \left[ \frac{p}{q} \right]^\frac{1-\alpha}{2}-1\bigg) 
    \Leftrightarrow f'(x)
    = \eta \frac{2}{1-\alpha}(1 - x^{\frac{\alpha-1}{2}}).
\end{equation}
Integrating both sides,
\begin{equation}
    f(x) = \eta \bigg(-\frac{4}{1-\alpha^2}x^\frac{\alpha+1}{2} + \frac{2}{1-\alpha}x \bigg) + C.
\end{equation}
Since $f(1)=0$ as $f$ induces an $f$-divergence, we obtain $C = \frac{2 \eta}{\alpha+1}$. The theorem then follows from
\begin{equation}
    f(x) 
    = \eta \bigg(-\frac{4}{1-\alpha^2}x^\frac{\alpha+1}{2} + \frac{2}{1-\alpha}x + \frac{2}{1+\alpha} \bigg) 
    = \eta \left\{ f_\alpha(x) + \frac{2}{1-\alpha}(x-1) \right\}.
\end{equation}
\end{proof}

\uniquerewards*
\begin{proof}
To be explicit,
\begin{equation}
    \bar{R}(p_\pi) \coloneqq (n+1)\int_\mathcal{S}p_\pi(s)\bar{r}(s)dV(s) = \int_{\mathcal{S}^{n+1}} \sum_{i=0}^n \bar{r}(s_i) d\mu(s_0)\prod_{i=1}^n dM(s_{i-1},s_i)
\end{equation}
where the equality holds by Lemma \ref{lemma:return_occupancy}.
The unique invariance under the agent-environment interaction when $p=p_\pi$ follows from the unique invariance of $p_\pi$ under the agent-environment interaction by Theorem \ref{theorem:ergodicity}.
For information monotonicity, we may consider $\bar{f}\left[p_\pi(s)\right]=f \left[ p_\pi(s)^{-1} \right]$ for some function $f$ without loss of generality, since the correspondence between $\bar{f}$ and $f$ is bijective.
The proof amounts to showing that $f$ satisfying information monotonicity is equivalent to strict concavity of $f$. Define $P'\coloneqq h_\kappa(P)$ and (abusing notation) $p'\coloneqq h_\kappa(p)$, recalling that $h_\kappa$ denotes the Markov morphism induced by the statistic $\kappa$.
We need to show that strictly concave $f$ uniquely satisfy
\begin{equation}
\label{equation:uniquerewards_coupledinequality}
    \int_\mathcal{S} f \left[ \frac{1}{p(s)} \right] dP(s) \overset{!}\leq \int_{\kappa(\mathcal{S})} f \left[ \frac{1}{p'(s)} \right] dP'(s) 
\end{equation}
with equality if and only if the statistic $\kappa$ is sufficient.
We use $!$ to indicate that a relation does not hold in general.
By \citep[Proposition 5.5-5.6]{ay2017information}, the statistic $\kappa$ is sufficient if and only if there exists $\omega:\mathcal{S}\rightarrow\mathbb{R}$ such that
\begin{equation}
    \label{equation:uniquerewards_sufficiency}
    p_\theta(s) \overset{!}{=} \omega(s)p_\theta'(\kappa(s))
\end{equation}
with $\omega$ independent of the probability density parameterization $\theta$.
If the statistic $\kappa$ is injective, we trivially have sufficiency of the statistic with equality in~\eqref{equation:uniquerewards_coupledinequality}. Assume therefore that the statistic $\kappa$ is not injective, such that at least one set of states $\Phi\subseteq\mathcal{S}$ with $|\Phi|>1$ is mapped to a single state $\kappa(\Phi)$. 
Note that, by definition, $P(\Phi)=P'(\kappa(\Phi))$,
Therefore, it suffices to consider the information monotonicity for the case of $\Phi$, since each such case corresponds to a certain amount of probability measure contributing additively to both sides of~\eqref{equation:uniquerewards_coupledinequality}.
Now, observe that
\begin{align}
    \label{equation:uniquerewards_ladder_1}
    \int_{\kappa(\Phi)} f\left[ \frac{1}{p'(s)} \right] dP'(s) 
    &= \int_{\kappa(\Phi)} f\left[ \frac{dV'(s)}{dP'(s)} \right] dP'(s) \nonumber \\
    &= P'(\kappa(\Phi)) f\left[ \frac{V'(\kappa(\Phi))}{P'(\kappa(\Phi))} \right] 
    = P(\Phi) f\left[ \frac{V(\Phi)}{P(\Phi)} \right] 
\end{align}
and notice that the argument of $f$ can be rewritten as
\begin{equation}
    \label{equation:uniquerewards_ladder_2}
    \frac{V(\Phi) }{ P(\Phi)}
    = \frac{\int_\Phi dV(s) }{ P(\Phi)}
    = \int_\Phi \frac{ p(s) }{ P(\Phi)} \frac{1}{p(s)} dV(s)
    = \int_\Phi \frac{1}{p(s)}  \frac{ dP(s) }{ P(\Phi)}
\end{equation}
Combining~\eqref{equation:uniquerewards_ladder_1} and~\eqref{equation:uniquerewards_ladder_2} with \eqref{equation:uniquerewards_coupledinequality}, we see that theorem follows from showing that strictly concave $f$ uniquely satisfy
\begin{equations}
    \int_\Phi f\left[  \frac{1}{p(s)} \right] dP(s) 
    &= P(\Phi) \int_\Phi f\left[  \frac{1}{p(s)} \right] \frac{ dP(s) }{ P(\Phi)} \\
    &\overset{!}{\leq}  P(\Phi) f\left[ \int_\Phi \frac{1}{p(s)} \frac{ dP(s) }{ P(\Phi)} \right]
    = \int_{\kappa(\Phi)} f\left[ \frac{1}{p'(s)} \right] dP'(s) 
\end{equations}
with equality if and only if $\kappa$ is a sufficient statistic.
Notice that satisfying the inequality with equality if and only if $p(s_1)^{-1}=p(s_2)^{-1}$ for any $s_1,s_2\in\Phi$, is equivalent to strict concavity by Jensen's inequality. Therefore, the theorem follows from showing that $p(s_1)=p(s_2)$ for any $s_1,s_2\in\Phi$ is equivalent to sufficiency of $\kappa$. Since our theorem must hold when $\mathcal{P}$ is the space of all probability measures, we may consider arbitrary parameterizations $\theta$ in~\eqref{equation:uniquerewards_sufficiency} without loss of generality. In particular, consider $s_1,s_2\in\Phi$ and $p_{\theta_1},p_{\theta_2}\in\mathcal{P}$ such that $p_{\theta_1}(s_2)=p_{\theta_2}(s_1)$. Then
\begin{equation}
    \frac{p_{\theta_1}(s_1)}{p'_{\theta_1}(\kappa(s_1))} 
    \overset{!}{=} \frac{p_{\theta_2}(s_1)}{p'_{\theta_2}(\kappa(s_1))}
    = \frac{p_{\theta_1}(s_2)}{p'_{\theta_1}(\kappa(s_2))}
    \Leftrightarrow p_{\theta_1}(s_1) \overset{!}{=} p_{\theta_1}(s_2)
\end{equation}
since $p'_{\theta_1}(\kappa(s_1))=p'_{\theta_1}(\kappa(s_2))$, which shows that $\kappa$ is sufficient if and only if $p(s_1)=p(s_2)$ for any $s_1,s_2\in\Phi$.
\end{proof}

\equivalencecountbased*
\begin{proof}
We first prove the case $\alpha=0$. From the definition of $\alpha$-information,
\begin{equation}
    I_0(s;p) \coloneqq I_{f_{\alpha=0}}(s;p) = 4\left(\sqrt{\frac{1}{p(s)}}-1\right).
\end{equation}
For a finite state space $\mathcal{S}$, we have
\begin{equation}
    p_\pi(s) = \frac{n(s)}{n+1}
\end{equation}
where $n(s)$ is the state count for state $s$ and $(n+1)$ is the total state count. The result then follows from
\begin{equation}
    \label{equation:hellinger_count}
    I_0(s,p_\pi) = \sqrt{\frac{16(n+1)}{n(s)}} -4,
\end{equation}
since $\sqrt{16(n+1)}$ is a constant which can be absorbed into $\beta$ and $4$ is a constant that disappears under gradients. 

We now prove the case $\alpha=-1$.
From Lemma \ref{lemma:return_occupancy}, we have
\begin{equations} 
    \!\!\!\! R_{\alpha,\beta}(\pi) 
    &= \int_{\mathcal{S}^{n+1}}R_{\alpha,\beta}(s;p_\pi) d\mu(s_0)\prod_{i=1}^n dM(s_{i-1},s_i) \\
    &= \int_{\mathcal{S}^{n+1}} \left\{ r(s) + \beta I_\alpha(s,p_\pi) \right\} d\mu(s_0)\prod_{i=1}^n dM(s_{i-1},s_i)  \\
    &= \int_{\mathcal{S}^{n+1}} r(s) d\mu(s_0)\prod_{i=1}^n dM(s_{i-1},s_i) + \beta \int_{\mathcal{S}^{n+1}} I_\alpha(s,p_\pi) d\mu(s_0)\prod_{i=1}^n dM(s_{i-1},s_i)  \\
    &= R(\pi) + \beta \int_{\mathcal{S}^{n+1}}  I_\alpha(s,p_\pi) d\mu(s_0)\prod_{i=1}^n dM(s_{i-1},s_i) \\
    \label{equation:equivalence_maximum_entropy_separation}
    &= R(\pi) + \beta(n+1) \int_\mathcal{S} p_\pi(s) I_\alpha(s,p_\pi) dV(s) 
\end{equations}
Now,
\begin{equation}
    \lim_{\alpha\rightarrow-1} \int_\mathcal{S} p_\pi(s)I_\alpha(s,p_\pi) dV(s)
    = \int_\mathcal{S} p_\pi(s) \log \frac{1}{p_\pi(s)} dV(s) \\
    \label{equation:equivalence_maximum_entropy_shannon}
    = H(p_\pi).
\end{equation}
Using (\ref{equation:equivalence_maximum_entropy_shannon}) and (\ref{equation:equivalence_maximum_entropy_separation}), the statement then follows from
\begin{equations}
  \lim_{\alpha\rightarrow-1} R_{\alpha,\beta}(\pi) 
  &= R(\pi) + \beta(n+1) H(p_\pi)
\end{equations}
since $(n+1)$ is a constant that can be absorbed into $\beta$.
\end{proof}

\optimadivergence*
\begin{proof}
From Lemma \ref{lemma:return_occupancy}, we have using the same derivation as~\eqref{equation:equivalence_maximum_entropy_separation},
\begin{equations}
    R_{\alpha,\beta}(p) 
    &= R(\pi) + \beta(n+1) \int_\mathcal{S} p_\pi(s) I_\alpha(s,p_\pi) dV(s) \\  
    &= R(\pi) - \beta(n+1)u^{-\frac{\alpha+1}{2}} \mathcal{D}_\alpha(p\| u) + \beta(n+1)(1-u^{-\frac{\alpha+1}{2}})
\end{equations}
Then, the optima
\begin{equations}
    p_{\alpha,\beta} 
    &= \argmax_{p\in\mathcal{P}_\Pi} R_{\alpha,\beta}(p) \\
    &= \argmax_{p\in\mathcal{P}_\Pi} \left\{ R(p) - \beta(n+1)u^{-\frac{\alpha+1}{2}} \mathcal{D}_\alpha(p\| u) \right\} \\
    &= \argmin_{p\in\mathcal{P}_\Pi} \left\{ \mathcal{D}_\alpha(p\| u) - \frac{u^\frac{\alpha+1}{2}}{\beta(n+1)}R(p) \right\} \\
    \label{equation:optimadivergence_partial}
    &= \argmin_{p\in\mathcal{P}_\Pi} \left\{ \mathcal{D}_\alpha(p\| u) - \frac{u^\frac{\alpha+1}{2}}{\beta(n+1)}[R(p) - c] \right\}
\end{equations}
with $c\in\mathbb{R}$ a constant. Interpreting $-\frac{u^\frac{\alpha+1}{2}}{\beta(n+1)}$ as a Lagrange multiplier, we have
\begin{equation}
    p_{\alpha,\beta} 
    = \argmin_{p\in\mathcal \{ p \in \mathcal{P} : R(p) = c_{\alpha,\beta} \}} \mathcal{D}_\alpha(p\| u) 
    = \argmin_{p\in \mathcal{H}(c_{\alpha,\beta})} \mathcal{D}_\alpha(p\| u).
\end{equation}
Note that we used convexity of $\mathcal{P}_\Pi$ to ensure strong duality, establishing the correspondence between the Lagrangian and the constrained optimization problem. 
\end{proof}

\optimaprojections*
\begin{proof}
The result follows from Lemma \ref{lemma:optima_divergence} and the $\alpha$-projection theorem \citep[Theorem 4.6]{amari2016information}.
\end{proof}

\optimaform*
\begin{proof}
From~\eqref{equation:optimadivergence_partial}, we have
\begin{equations}
    p_{\alpha,\beta} 
    &= \argmin_{p\in\mathcal{P}} \left\{ \mathcal{D}_\alpha(p\| u) - \frac{u^\frac{\alpha+1}{2}}{\beta(n+1)}[R(p) - c] \right\} \\
    &= \argmin_{p\in\mathcal{P}} \left\{ -\frac{4}{1-\alpha^2}\int p(s)^\frac{1-\alpha}{2} dV(s) - \frac{1}{\beta} \Big(\int p(s)r(s) dV(s) - c\Big) \right\}
\end{equations}
These distributions are determined by minimizing a Lagrangian functional, composed of an entropy function and constrained by a fixed expected reward and a normalization condition.
\begin{equation}
    L[p] = -\frac{4}{1-\alpha^2}\int p^\frac{1-\alpha}{2} dV -\frac{1}{\beta} \Big(\int pr dV - c\Big)+ \lambda \Big(\int p dV - 1\Big)
\end{equation}
where $\lambda$ is a Lagrange multiplier. The minima is obtained by enforcing $\delta L =0$, where
\begin{equation}
    \delta L[p]
    = \int \Big[ -\frac{2}{\alpha+1} p^{-\frac{\alpha+1}{2}} - \frac{r}{\beta} + \lambda \Big] \delta p\, dV = 0
\end{equation}
this requires that function inside the squared brackets vanishes, leading to
\begin{equation}
    \label{equation:lagrangian_first_optimality}
    p(s) = \left[ \frac{\lambda(\alpha+1)}{2}\right]^{-\frac{2}{\alpha+1}}
    \Big(1 - \frac{r(s)}{\lambda\beta}\Big)^{-\frac{2}{\alpha+1}}.
\end{equation}
From the normalization of the probability distribution~\eqref{equation:lagrangian_first_optimality}, $\int p dV = 1$, one finds 
\begin{equation}
    \label{equation:lagrangian_second_optimality}
    \left[ \frac{\lambda(\alpha+1)}{2} \right]^{-\frac{2}{\alpha+1}} = \frac{1}{\int (1 - \frac{r}{\lambda\beta})^{-\frac{2}{\alpha+1}}dV}
\end{equation}
where we used (\ref{equation:lagrangian_first_optimality}).
Substituting (\ref{equation:lagrangian_second_optimality}) into (\ref{equation:lagrangian_first_optimality}), we get
\begin{equation}
    \label{equation:optima_with_policy}
    p_{\alpha,\beta} = \frac{(1-\frac{r}{\lambda\beta})^{-\frac{2}{\alpha+1}}}{\int (1-\frac{r}{\lambda\beta})^{-\frac{2}{\alpha+1}}dV}
\end{equation}
from which the result follows since a reparameterization $\lambda\leftarrow-\lambda$ does not change the form.
\end{proof}

\optimageodesic*
\begin{proof}
We prove first the forward direction by cases, assuming $f=f_\alpha$.
Assume first that $\alpha\neq-1$. The $(\alpha+2)$-geodesic connecting $p$ and $q$ is given by 
\begin{equation}
    \gamma_{p,q}(t) = \left\{ (1-t) p^{-\frac{1+\alpha}{2}} + t q^{-\frac{1+\alpha}{2}} \right\}^{-\frac{2}{1+\alpha}} \xi(t)
\end{equation}
where $t\in [0,1]$ and $\xi(t)$ normalizes $\gamma_{p,q}$ at $t$~\citep[Equation 2.59]{ay2017information}. 
Let $\beta_0,\beta_1\in(0,\infty)$ parameterize two points $p_0 = p_{\alpha,\beta_0}$ and $p_1 = p_{\alpha,\beta_1}$ connected by $\gamma_{p_0,p_1}$. By Lemma~\ref{lemma:optima_form}, $p_{0,1}^{-\frac{1+\alpha}{2}} = \xi_{0,1} (1+ \mu_{0,1}r)$ 
\begin{equation}
    \gamma_{p_{0},p_{1}}(t) 
    = \left\{ (1-t)(1+\mu_0 r)\xi_0 + t(1+\mu_1 r)\xi_1 \right\}^{-\frac{2}{1+\alpha}} \xi(t),
\end{equation}
with $\mu_{0,1} = \frac{1}{\lambda\beta_{0,1}}$ adopted for brevity. Grouping the terms linear in $r$,
\begin{equations}
    \gamma_{p_{0},p_{1}}(t)
    &= \left\{ 1 + \left[ \left( 1-\frac{t\xi_1}{(1-t)\xi_0 + t\xi_1} \right) \mu_0 + \frac{t\xi_1}{(1-t)\xi_0 + t\xi_1}\mu_1 \right] r \right\}^{-\frac{2}{1+\alpha}} \nonumber \\
    &\hphantom{=} \quad [ (1-t)\xi_0 + t\xi_1 ]^{-\frac{2}{1+\alpha}} \xi(t) \\
    &= \left\{ 1 + \left[ (1-\hat{t})\mu_0 + \hat{t}\mu_1 \right] r \right\}^{-\frac{2}{\alpha+1}} [ (1-t)\xi_0 + t\xi_1 ]^{-\frac{2}{\alpha+1}} \xi(t) \\
    \label{equation:normalization}
    &= \left\{ 1 + \left[ (1-\hat{t})\mu_0 + \hat{t}\mu_1 \right] r \right\}^{-\frac{2}{\alpha+1}} \hat{\xi}(\hat{t})
\end{equations}
where $\hat{t} \coloneqq\frac{t\xi_1}{(1-t)\xi_0+t\xi_1}$ and $\hat{\xi}(\hat{t})$ ensures normalization of~\eqref{equation:normalization}.
\begin{equations}
    \gamma_{p_{0},p_{1}}(\hat{t})^{-\frac{\alpha+1}{2}} &= \left\{ 1 + \left[ (1-\hat{t})\mu_0 + \hat{t}\mu_1 \right] r \right\} \hat{\xi}(\hat{t}) \\
    &= \left\{ 1 + \left[ (1-\hat{t})\frac{1}{\lambda\beta_0} + \hat{t}\frac{1}{\lambda\beta_1} \right] r \right\} \xi(t)
\end{equations}
leading to
\begin{equation}
    \gamma_{p_{0},p_{1}}(\hat{t}) = \left\{ 1 +  \left[ (1-\hat{t})\frac{1}{\lambda\beta_0} + \hat{t}\frac{1}{\lambda\beta_1} \right] r \right\} ^{-\frac{2}{\alpha+1}}\phi(t)  
\end{equation}
where $\phi(t) \coloneqq \xi(t)^{-\frac{2}{\alpha+1}}$ normalizes $\gamma_{p_{0},p_{1}}(\hat{t})$. Then,
\begin{equations}
    \gamma_{p_{0},u}(t) 
    &= \lim_{\beta_1\rightarrow\infty} \gamma_{p_{0},p_{1}}(t) \\  
    &= \left[1 +  \frac{(1-\hat{t})}{\lambda\beta_0}r\right]^{-\frac{2}{\alpha+1}}\phi(t) \\  
    &= p_{\alpha,\beta_t}
\end{equations}
where $\beta_t\coloneqq\frac{\beta_0}{1-\hat{t}}$. Therefore,
\begin{equation}
    \left\{ \gamma_{p_{0},u}(t) : t\in(0,1) \right\} 
    \subseteq \left\{ p_{\alpha,\beta} : \beta \in (\beta_0,\infty) \right\}. 
    \label{equation:corollary_forward}
\end{equation}
For any $\beta>\beta_0$, there exists a $\bar{t}\in(0,1)$ by the intermediate value theorem such that
\begin{equation}
    p_{\alpha,\beta} = \gamma_{p_{0},u}(\bar{t})
\end{equation}
since $t\mapsto\beta_t$ is a continuous map between the connected intervals $[0,1)$ and $[\beta_0,\infty)$. Thus,
\begin{equation}
    \label{equation:corollary_backward}
    \left\{ p_{\alpha,\beta} : \beta \in (\beta_0,\infty) \right\}
    \subseteq \left\{ \gamma_{p_{0},u}(t) : t\in(0,1) \right\}. 
\end{equation}
Taking the closure of~\eqref{equation:corollary_forward} and~\eqref{equation:corollary_backward}, we have the result for $\alpha\neq-1$:
\begin{equation}
    \left\{ \gamma_{p_{0},u}(t) : t\in[0,1] \right\} 
    \subseteq \overline{ \left\{ p_{\alpha,\beta} : \beta \in (\beta_0,\infty) \right\} }
    \subseteq \left\{ \gamma_{p_{0},u}(t) : t\in[0,1] \right\}.
\end{equation}
For $\alpha=-1$, we have from a reparameterization $\lambda\leftarrow-\frac{2\lambda}{\alpha+1}$ of the optima in Lemma \ref{lemma:optima_form} (the reparameterization does not change the form) and the definition of Euler's number $e^x\coloneqq\lim_{n\rightarrow\infty}(1+\frac{x}{n})^n$,
\begin{equations}
    \lim_{\alpha\rightarrow-1} p_{\alpha,\beta}
    &= \lim_{\alpha\rightarrow-1} \frac{(1-\frac{\alpha+1}{2\lambda\beta}r)^{-\frac{2}{\alpha+1}}}{\int (1-\frac{\alpha+1}{2\lambda\beta}r)^{-\frac{2}{\alpha+1}}dV} \\
    &= \frac{e^{\frac{r}{\lambda\beta}}}{\int e^{\frac{r}{\lambda\beta}} dV}.
\end{equations}
The forward direction result follows from a similar argument to the $\alpha\neq-1$ case.
\\\\
For the backward direction, assume that $\beta\mapsto p_{f,\beta}$ is an $(\alpha+2)$-geodesic. Let $\beta'>\beta$ and define
\begin{equation}
  \bar{\mathcal{D}}_f(\beta) \coloneqq D_f(p_{f,\beta},u)
\end{equation}
By optimality of $p_{f,\beta}$ and $p_{f,\beta'}$, we have the inequalities
\begin{equations}
  R(p_{f,\beta}) - \beta \bar{\mathcal{D}}_f(\beta) &> R(p_{f,\beta'}) - \beta \bar{\mathcal{D}}_f(\beta') \\
  R(p_{f,\beta'}) - \beta' \bar{\mathcal{D}}_f(\beta') &> R(p_{f,\beta}) - \beta' \bar{\mathcal{D}}_f(\beta)
\end{equations}
Adding the inequalities and using $\beta'>\beta$, we get
\begin{equations}
   R(p_{f,\beta}) - \beta \bar{\mathcal{D}}_f(\beta) + R(p_{f,\beta'}) - \beta' \bar{\mathcal{D}}_f(\beta') &> R(p_{f,\beta'}) - \beta \bar{\mathcal{D}}_f(\beta') + R(p_{f,\beta}) - \beta' \bar{\mathcal{D}}_f(\beta) \\
  \therefore 
  - \beta \bar{\mathcal{D}}_f(\beta) - \beta' \bar{\mathcal{D}}_f(\beta') &> - \beta \bar{\mathcal{D}}_f(\beta') - \beta' \bar{\mathcal{D}}_f(\beta) \\
  \therefore 
  \bar{\mathcal{D}}_f(\beta) &> \bar{\mathcal{D}}_f(\beta') 
\end{equations}
which shows that $\bar{\mathcal{D}}_f$ is strictly monotonic decreasing. By a similar argument, $\bar{\mathcal{D}}_\alpha(\beta) \coloneqq D_\alpha(p_{f_\alpha,\beta},u)$ is strictly monotonic decreasing. Since $\bar{\mathcal{D}}_f$ is strictly monotonic decreasing, its inverse $\bar{\mathcal{D}}_f^{-1}$ exists and is strictly monotonic decreasing. Then, for any $p$ along the $(\alpha+2)$-geodesic,
\begin{equation}
  D_\alpha(p,u) = (\bar{\mathcal{D}}_\alpha \circ \bar{\mathcal{D}}_f^{-1}) \circ D_f(p,u), \quad p\in \gamma^{\alpha+2}_{p^*,u}
\end{equation}
Since $(\bar{\mathcal{D}}_\alpha\circ\bar{\mathcal{D}}_f^{-1})$ is strictly monotonic increasing, $D_\alpha(\cdot,u)$ is a strictly monotonic increasing function of $D_f(\cdot,u)$ when restricted to the $(\alpha+2)$-geodesic.
But then, this restricted $D_f$ is geodetic by the proof of Lemma \ref{lemma:geodetic}. Furthermore, since the proof of Lemma \ref{lemma:alphaunique} requires only the geodetic property along some $\alpha$-geodesic with distinct endpoints, we have $f=f_\alpha$.
\end{proof}
\optimapath*
\begin{proof}
  It is a standard result that if $P_\Pi$ is convex, a local optimum $p_{f,\beta}\in\mathcal{P}_\Pi$ is the unique global optimum, such that $\beta\mapsto p_{f,\beta}$ is a map.
  The result follows from this and Berge's maximum theorem. We give a simplified proof under convexity, following \cite{berge1959espaces}.
  We note that $(\beta,p)\mapsto R_{f,\beta}(p)$ is continuous in both $\beta$ and $p$. Let $\beta\in\mathbb{R}_{\geq 0}$ and $\varepsilon > 0$. We first show that $\gamma_f$ is lower semi-continuous. Since $(\beta,p)\mapsto R_{f,\beta}(p)$ is lower semi-continuous, there exists a neighborhood $(U,V)\subseteq \mathbb{R}_{\geq 0} \times \mathcal{P}_\Pi$ of $(\beta, \argmax_{p\in\mathcal{P}_\Pi} R_{f,\beta}(p))$ such that for any $(\beta',p')\in (U,V) $,
  \begin{equation}
    \max_{p\in\mathcal{P}_\Pi} R_{f,\beta}(p) < R_{f,\beta'}(p') + \varepsilon
  \end{equation}
Lower semi-continuity of $\gamma_f$ then follows from
  \begin{equation}
    \max_{p\in\mathcal{P}_\Pi} R_{f,\beta}(p) < R_{f,\beta'}(p') + \varepsilon \leq \max_{p\in\mathcal{P}_\Pi} R_{f,\beta'}(p) + \varepsilon
  \end{equation}
Next, we show that $\gamma_f$ is upper semi-continuous.
By upper semi-continuity of $R_{f,\beta}$, for each $p\in\mathcal{P}_\Pi$ there exists a neighborhood $(U_p,V_p)$ such that for any $(\beta',p')\in(U_p,V_p)$,
\begin{equation}
  R_{f,\beta'}(p') < R_{f,\beta}(p) + \varepsilon
\end{equation}
The sets $\{ V_p : p \in \mathcal{P}_\Pi \}$ are an open cover of $\mathcal{P}_\Pi$. Since $\mathcal{P}_\Pi$ is compact, there exists a finite subcover $\{ V_{p_k} : k \in (1,...,n) \}$. Now, for each $\beta'\in \bigcap_k U_{p_k}$, we have for any $p'\in\mathcal{P}_\Pi$,
\begin{equation}
  R_{f,\beta'}(p') < R_{f,\beta}(p_k) + \varepsilon
\end{equation}
for some $k$, since $p' \in V_{p_k}$ for some $k$. Upper semi-continuity of $\gamma_f$ then follows from
\begin{equation}
  \max_{p'\in\mathcal{P}_\Pi} R_{f,\beta'}(p') < \max_{k\in(1,...,n)} R_{f,\beta}(p_k) + \varepsilon \leq \max_{p\in\mathcal{P}_\Pi} R_{f,\beta}(p) + \varepsilon
\end{equation}
Since $\gamma_f$ is both lower and upper semi-continuous, it is continuous.

\end{proof}

\equivalencerenyi*
\begin{proof}
  Recall that Rényi maximum entropy exploration uses \eqref{equation:maximum_entropy} with Rényi entropy
  \begin{equation}
    H_\lambda(p) \coloneqq \frac{1}{1-\lambda} \log \int_\mathcal{X}p(x)^\lambda dP(x), \quad \lambda \in \mathbb{R}_{\geq0}
  \end{equation}
  instead of Shannon entropy $H$.
  Observe the relationship between Rényi entropy and the Rényi divergence
  \begin{equation}
    D_\lambda(p,u) 
    \coloneqq \frac{1}{\lambda-1}\log\int_\mathcal{S} p(s)^\lambda u^{1-\lambda}dV(s) 
    = - H_\lambda(p) - \log u
  \end{equation}
  Since the $\alpha$-divergences and Rényi divergences are monotonically related by the Box-Cox transform $F$,
  \begin{equation}
    D_\alpha(p,u) = F(D_\lambda(p,u)),
  \end{equation}
  we have
  \begin{equations}
    \int_{\mathcal{S}} p(s) I_{\alpha}(s;p) dV(s)
    &= - u^{-\frac{\alpha+1}{2}} D_\alpha(p,u) - \frac{4}{1-\alpha^2}(1 - u^{-\frac{\alpha+1}{2}} ) \\
    &= - u^{-\frac{\alpha+1}{2}} F(D_\lambda(p,u)) - \frac{4}{1-\alpha^2}(1 - u^{-\frac{\alpha+1}{2}} )
  \end{equations}
  Now, using Lemma \ref{lemma:return_occupancy},
  \begin{equations}
    R_{\alpha,\beta}(\pi) 
    &= R(\pi) + \beta(n+1) \int_{\mathcal{S}} p_\pi(s) I_{\alpha}(s;p_\pi) dV(s) \\
    &= R(\pi) -\beta(n+1) u^{-\frac{\alpha+1}{2}} F(D_\lambda(p,u)) - \frac{4 \beta(n+1)}{1-\alpha^2}(1 - u^{-\frac{\alpha+1}{2}})
  \end{equations}
  For a parameterized policy $\pi_\theta$, we then have
  \begin{equations}
    \partial_i R_{\alpha,\beta}(\pi_\theta)
    &= \partial_i \left\{ R(\pi_\theta) -\beta(n+1) u^{-\frac{\alpha+1}{2}} F(D_\lambda(p_{\pi_\theta},u)) - \frac{4 \beta(n+1)}{1-\alpha^2}(1 - u^{-\frac{\alpha+1}{2}}) \right\} \\
    &= \partial_i R(\pi_\theta) - \beta(n+1)u^{-\frac{\alpha+1}{2}} F'(D_\lambda(p_{\pi_\theta},u) \partial_i D_\lambda (p_{\pi_\theta},u) \\
    &= \partial_i R(\pi_\theta) + \beta(n+1)u^{-\frac{\alpha+1}{2}} F'(D_\lambda(p_{\pi_\theta},u) \partial_i H_\lambda(p_{\pi_\theta}) \\
    &= \partial_i R(\pi_\theta) + C(\theta) \beta \partial_i H_\lambda(p_{\pi_\theta}),
  \end{equations}
  with $C(\theta) \coloneqq (n+1)u^{-\frac{\alpha+1}{2}} F'(D_\lambda(p_{\pi_\theta},u)$, from which the result follows.
\end{proof}

\equivalenceclassic*
\begin{proof}
 Classic artificial curiosity uses
\begin{equation}
  r(s) + \beta \| s - f_\theta(s_0, \pi(s_0)) \|^2
\end{equation}
where $f_\theta$ is a neural network model of the transition function. By hypothesis, we have
\begin{equation}
  M^1(s_0,s) 
  = (2\pi)^\frac{d}{2} \exp(-\frac{1}{2}(s - f_\theta(s_0))^\top (s - f_\theta(s_0))),
\end{equation}
Since $I_{-1}(s;p) \coloneqq \lim_{\alpha\rightarrow-1}I_\alpha(s;p)= -\log p(s)$, then 
\begin{equations}
  I_{-1}(s;M^1(s_0,\cdot)) 
  &= -\log ((2\pi)^\frac{d}{2}\exp(-\frac{1}{2}(s-f_\theta(s_0))^\top (s-f_\theta(s_0)))) \\
  &= \frac{1}{2} (s-f_\theta(s_0))^\top (f_\theta(s)-s_0) - \frac{d}{2}\log 2\pi
\end{equations}
The statement follows since $\frac{1}{2}$ can be absorbed into $\beta$ and $\frac{d}{2}\log 2\pi$ vanishes under gradients.
\end{proof}

\end{document}